\documentclass{article}
\pdfoutput=1
\usepackage{arydshln}
\usepackage{diagbox}
\usepackage{float}
\usepackage{bm}
\usepackage{subcaption}

\usepackage{microtype}
\usepackage{graphicx}
\usepackage{booktabs} 

\usepackage{hyperref}

\usepackage[accepted]{icml2024}

\usepackage{amsmath}
\usepackage{amssymb}
\usepackage{mathtools}
\usepackage{amsthm}
\usepackage[bottom]{footmisc}

\usepackage[capitalize,noabbrev]{cleveref}

\theoremstyle{plain}
\newtheorem{theorem}{Theorem}[section]

\newtheorem{lemma}[theorem]{Lemma}
\newtheorem{corollary}[theorem]{Corollary}
\theoremstyle{definition}
\newtheorem{definition}[theorem]{Definition}
\newtheorem{assumption}[theorem]{Assumption}
\theoremstyle{remark}

\DeclareMathOperator{\Lp}{Lip}

\usepackage[textsize=tiny]{todonotes}
\usepackage{xcolor}
\usepackage{tcolorbox}
\usepackage{mdframed}
\usepackage{bbding}
\usepackage{wrapfig}
\usepackage{amssymb}
\definecolor{wbydarkgreen}{RGB}{0, 100, 0}

\newmdenv[
backgroundcolor=gray!10, 
linecolor=black, 
roundcorner=10pt, 
frametitlebackgroundcolor=blue!20, 
frametitlefont=\sffamily\bfseries, 
frametitleaboveskip=10pt, 
innertopmargin=\topskip, 
skipabove=10pt, 
skipbelow=10pt, 
]{mybox}
\usepackage{listings}

\icmltitlerunning{LLM-Empowered State Representation for Reinforcement Learning}

\begin{document}

\twocolumn[
\icmltitle{LLM-Empowered State Representation for Reinforcement Learning}



\icmlsetsymbol{equal}{*}

\begin{icmlauthorlist}
\icmlauthor{Boyuan Wang}{equal,thu}
\icmlauthor{Yun Qu}{equal,thu}
\icmlauthor{Yuhang Jiang}{thu}
\icmlauthor{Jianzhun Shao}{thu}
\icmlauthor{Chang Liu}{thu}
\icmlauthor{Wenming Yang}{thu}
\icmlauthor{Xiangyang Ji}{thu} 
\end{icmlauthorlist}

\icmlaffiliation{thu}{Tsinghua University}
\icmlcorrespondingauthor{Xiangyang Ji}{xyji@tsinghua.edu.cn}

\icmlkeywords{Machine Learning, ICML}

\vskip 0.3in
]



\printAffiliationsAndNotice{\icmlEqualContribution} 
\begin{abstract}

Conventional state representations in reinforcement learning often omit critical task-related details, presenting a significant challenge for value networks in establishing accurate mappings from states to task rewards.
Traditional methods typically depend on extensive sample learning to enrich state representations with task-specific information, which leads to low sample efficiency and high time costs.  
Recently, surging knowledgeable large language models (LLM) have provided promising substitutes for prior injection with minimal human intervention.
Motivated by this, we propose LLM-Empowered State Representation~(LESR), a novel approach that utilizes LLM to autonomously generate task-related state representation codes which help to enhance the continuity of network mappings and facilitate efficient training.
Experimental results demonstrate LESR exhibits high sample efficiency and outperforms state-of-the-art baselines by an average of \textbf{29\%} in accumulated reward in Mujoco tasks and \textbf{30\%} in success rates in Gym-Robotics tasks. 
Codes of LESR are accessible at \url{https://github.com/thu-rllab/LESR}.

\end{abstract} 
\section{Introduction}\label{section:intro}

Traditional reinforcement learning (RL) algorithms~\cite{tesauro1995temporal, watkins1992q, rummery1994line} generally require a large number of samples to converge~\cite{maei2009convergent, ohnishi2019constrained}.
Compounding this challenge, in most cases, the intricate nature of RL tasks significantly hampers sample efficiency.~\cite{dulac2019challenges}. For handling complex tasks and enhancing generalization, neural networks are utilized to approximate value functions~\cite{mnih2015human, schulman2015high}.
However, value networks often lack smoothness even when they are trained to converge, leading to instability and low sample efficiency throughout the training process~\cite{asadi2018lipschitz}.
Considering that in deep RL, state vectors serve as the primary input to value networks, sub-optimal state representations can result in limited generalization capabilities and non-smoothness of value network mappings~\cite{schuck2018state, merckling2022exploratory}.

In most RL environments~\cite{todorov2012mujoco, brockman2016openai, de2023gymnasium}, source state representations typically embody environmental information but lack specific task-related details.
However, the learning of value networks depends on accurately capturing task dynamics through rewards~\cite{yang2020multi, yoo2022skills}.
The absence of task-related representations may impede the establishment of network mappings from states to rewards, affecting network continuity.
Previous researches leverage diverse transitions, intensifying the issue of time-consuming data collection and training~\cite{merckling2020state, sodhani2021multi, merckling2022exploratory}.
Consequently, a question arises regarding the existence of a more efficient method for identifying task-related state representations.

Indeed, incorporating task-specific information into the state representation can be achieved through introducing expert knowledge~\cite{niv2019learning, ota2020can}.
Recently, significant strides have been accomplished in the field of Large Language Models (LLM)~\cite{touvron2023llama, openai2023gpt4}.
The exceptional performance of LLM in various domains~\cite{shen2023pangu, miao2023selfcheck, liang2023encouraging, madaan2023self, stremmel2023xaiqa}, particularly in sequential decision-making tasks, showcases their extensive knowledge and sufficient reasoning abilities~\cite{wang2023voyager, wang2023describe, feng2023llama, wu2023spring, du2023guiding, shukla2023lgts, sun2023prompt, zhang2023rladapter}.
This motivates the idea that the prior knowledge embedded in LLM can be exploited to enhance the generation of task-related state representations.

\begin{figure*}[t]
	\centering
	
	\begin{subfigure}{0.22\linewidth}
		\centering
		\includegraphics[width=\linewidth]{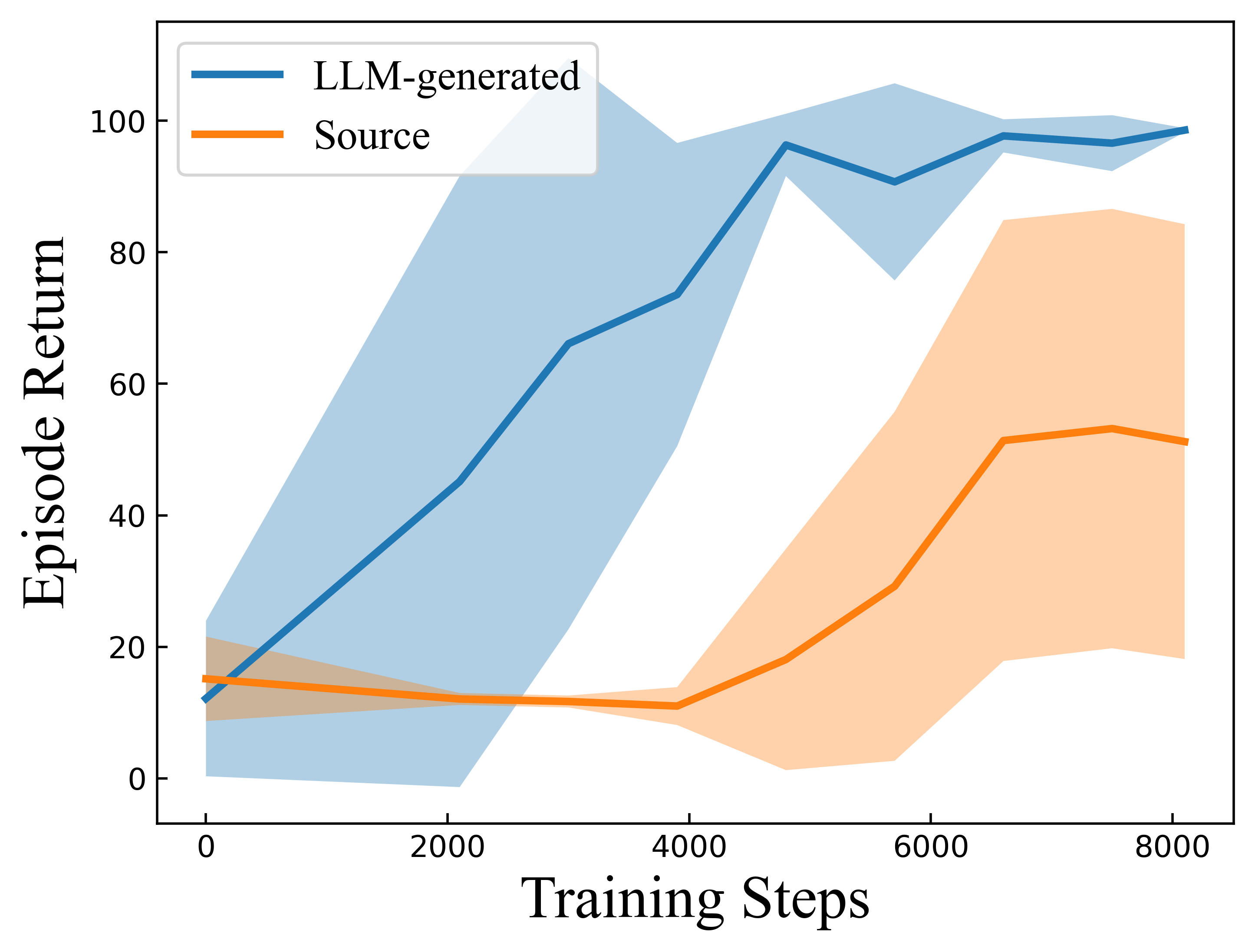}
		\captionsetup{font=footnotesize}
		\caption{\ \ Training Curves}
		\label{fig:demo-a}
	\end{subfigure}
	\hspace{10pt}
	\hspace{-10pt}
	\begin{subfigure}{0.4\linewidth}
		\centering
		\includegraphics[height=.4\linewidth]{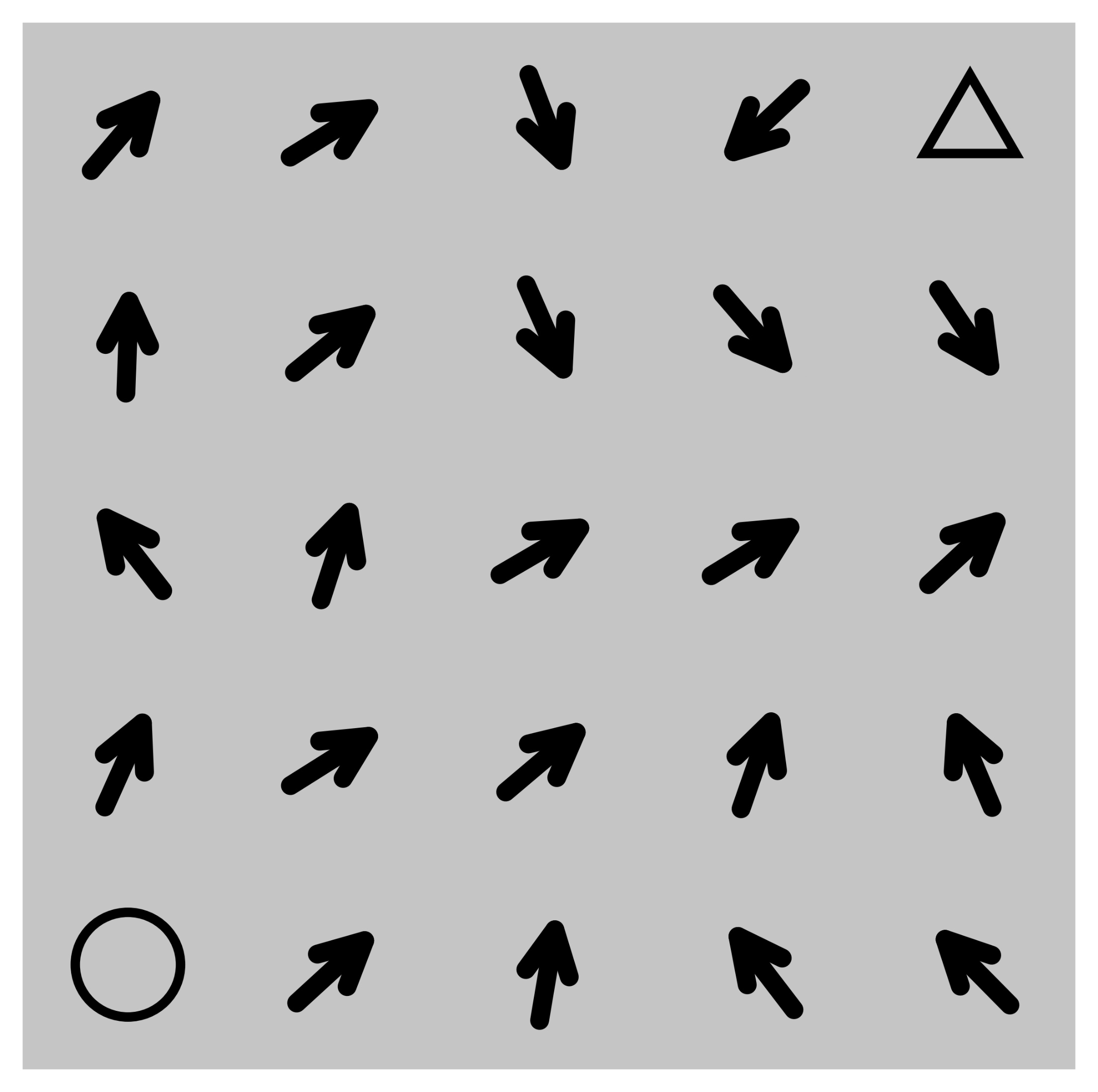}
		\centering
		\includegraphics[height=.4\linewidth]{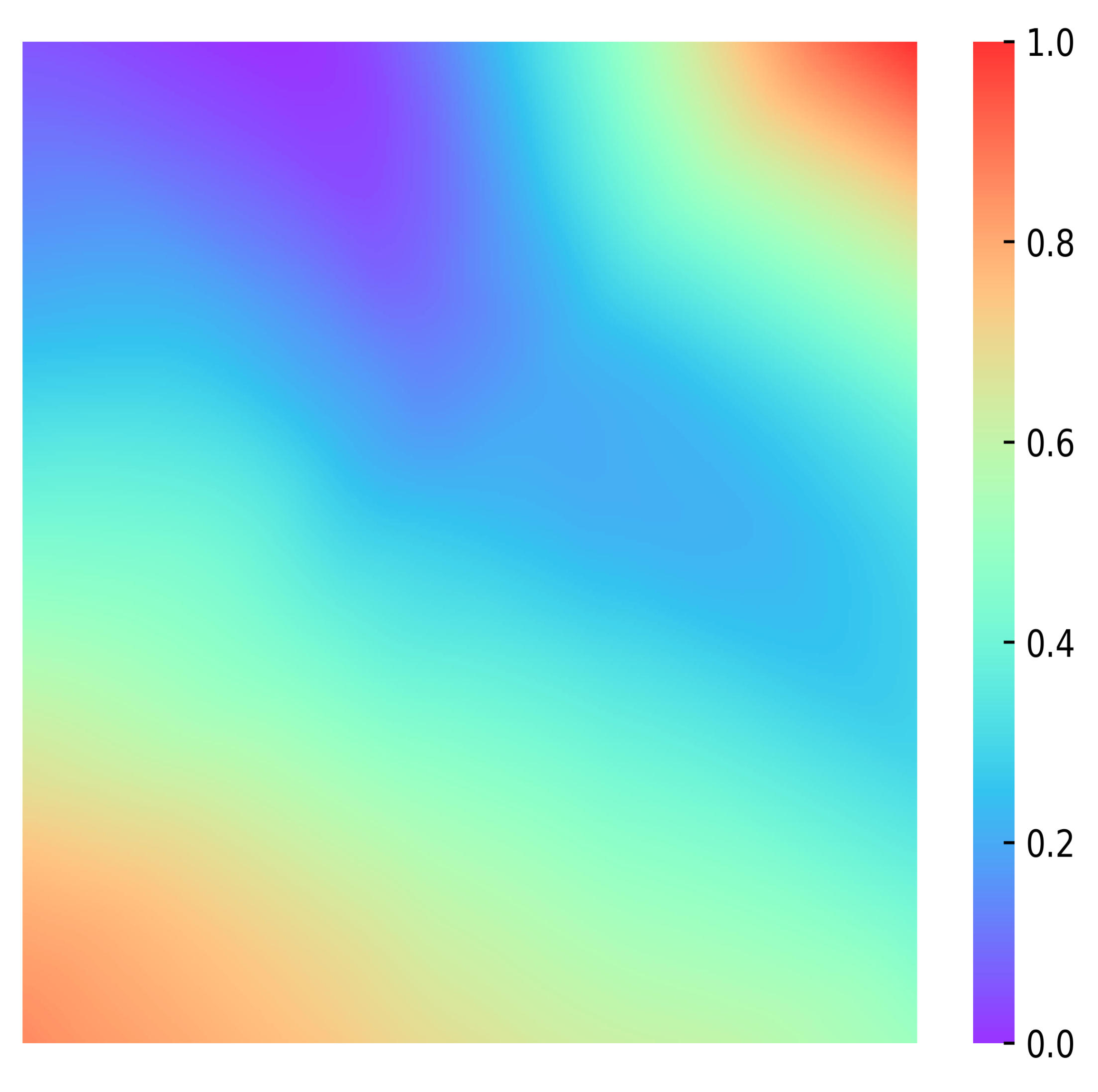}
		\vspace{5pt}
		\captionsetup{font=footnotesize}
		\caption{Source, $\Lp(Q) = 37.47$}
		\label{fig:demo-b}
	\end{subfigure}
	\hspace{-14pt}
	\hspace{-10pt}
	\begin{subfigure}{0.4\linewidth}
		\centering
		\includegraphics[height=.4\linewidth]{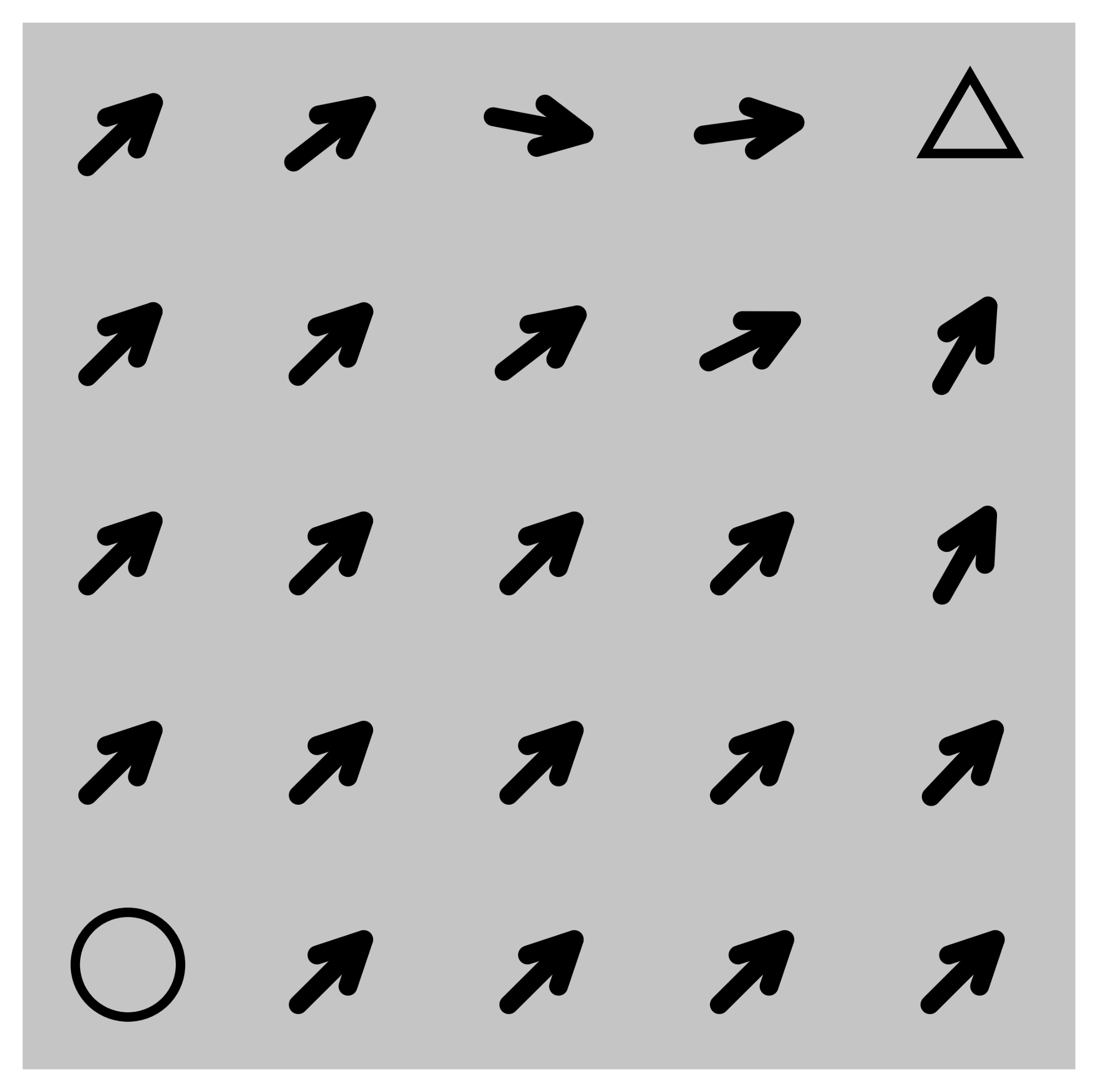}
		\centering
		\includegraphics[height=.4\linewidth]{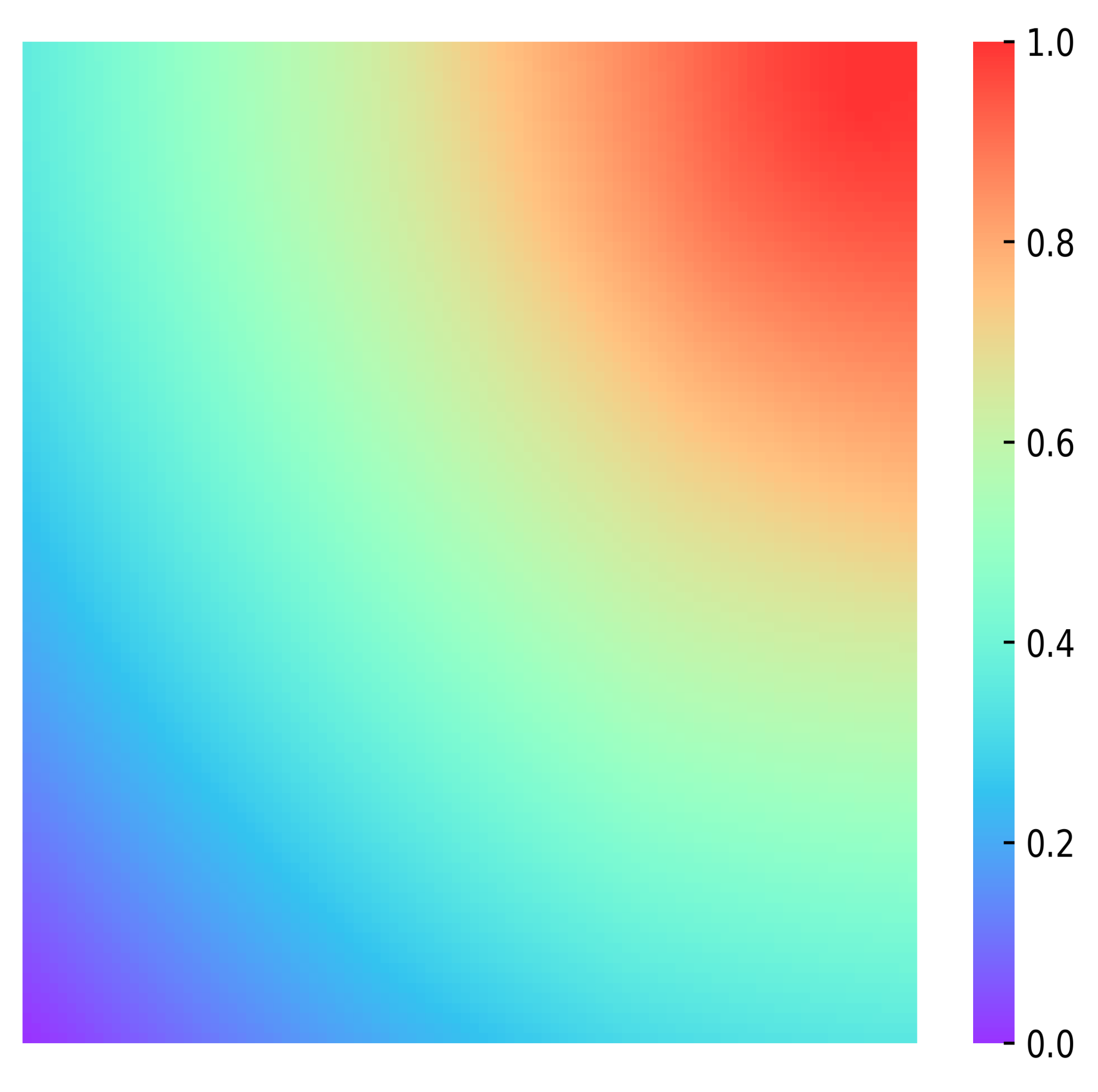}
		\vspace{5pt}
		\captionsetup{font=footnotesize}
		\caption{LLM-generated, $\Lp(Q) = 3.15$}
		\label{fig:demo-c}
	\end{subfigure}
	
	\caption{Experiments conducted in the \texttt{PointMaze} in Gym-Robotics~\cite{de2023gymnasium}. The agent aims to navigate from the bottom-left (depicted by `$\bigcirc$') to the top-right target (depicted by `$\triangle$'). (a) The training curves of policies in two different state representations: (1) Source: which only involves the coordinates of the agent and the target; (2) LLM-Generated: which adds another dimension indicating the distance from the agent to the target.
	(b) (c): Visualizations of final policies and learned state values. Arrows show actions by final policies. Heatmaps display learned state values after just 500 training steps, and the smoother result of LLM-generated shows higher sample efficiency.
	The Lipschitz constant Lip($Q$) is defined in Definition~\ref{defi:constant}.
	Further details are provided in Appendix~\ref{appendix:toy_example}.}
	\label{fig:demo}
\end{figure*} 

To preliminarily validate LLM's capacity in enhancing state representations, an illustrative toy example is provided. In Figure~\ref{fig:demo}, results demonstrated that the state representation generated by LLM can help to enhance the continuity of value networks and expedite the convergence of policy learning.
To substantiate this, we employ the Lipschitz constant~\cite{jones1993lipschitzian} for smoothness assessment.
The LLM-generated state representations enhance the Lipschitz continuity of value networks, accounting for the improvement of sample efficiency and performance.

In this paper, we propose a novel method named \textbf{L}LM-\textbf{E}mpowered \textbf{S}tate \textbf{R}epresentation (\textbf{LESR}).
We utilize LLM's coding proficiency and interpretive capacity for physical mechanisms to generate task-related state representation function codes. LLM is then employed to formulate an intrinsic reward function based on these generated state representations. A feedback mechanism is devised to iteratively refine both the state representation and intrinsic reward functions. In the proposed algorithm, LLM consultation takes place only at the beginning of each iteration. Throughout both training and testing stages, LLM is entirely omitted, ensuring significant time savings and flexibility.

In summary, our main contributions are:
\begin{itemize}
	\item We propose a novel method employing LLM to generate task-related state representations accompanied by intrinsic reward functions for RL. These functions are demonstrated to exhibit robustness when transferred to various underlying RL algorithms.
	\item We have theoretically demonstrated that enhancing Lipschitz continuity improves the convergence of the value networks and empirically validated more task-related state representations can enhance Lipschitz continuity. 
	\item LESR is a general framework that accommodates both continuous and discontinuous reward scenarios. Experimental results demonstrate that LESR significantly surpass state-of-the-art baselines by an average improvement of \textbf{29\%} in Mujoco tasks and \textbf{30\%} in Gym-Robotics tasks. We have also experimentally validated LESR's adaptability to novel tasks.
\end{itemize}
\section{Related Work}

\textbf{Incorporating LLM within RL Architecture} \quad Since the advent of LLM, researchers have endeavored to harness the extensive common-sense knowledge and efficient reasoning abilities inherent in LLMs within the context of RL environments. Challenges have arisen due to the misalignment between the high-level language outputs of LLMs and the low-level executable actions within the environment. To address this, \citet{qiu2023embodied, agashe2023evaluating, yuan2023plan4mc, wang2023voyager, wang2023describe, feng2023llama, wu2023spring} have sought to employ environments where observation and action spaces can be readily translated into natural language~\cite{carta2023grounding, puig2018virtualhome}. Alternative approaches employ language models as the policy network with fine-tuning~\cite{zhang2023rladapter, li2022pre, shi2023unleashing, yan2023ask, carta2023grounding}. Meanwhile, other endeavors focus on leveraging LLMs as high-level planners, generating sub-goals for RL agents~\cite{sun2023prompt, shukla2023lgts, zhang2023bootstrap}. Nevertheless, these works encounter a common challenge: the tight coupling of LLMs with RL agents, leading to frequent communication between the two even during the testing stage, a process that proves time-consuming and inefficient.

\begin{figure*}[t]
	\centering
	\includegraphics[width=\textwidth]{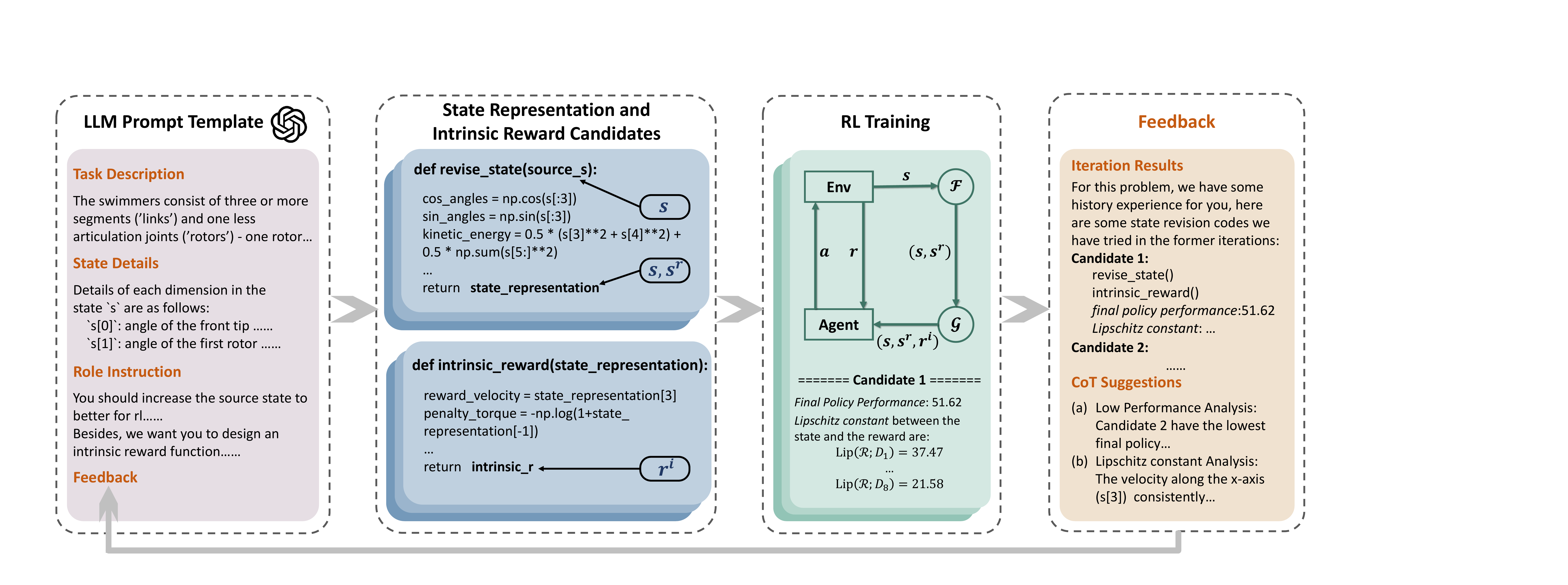}
	\caption{LESR Framework: (1) LLM is prompted to generate codes for state representation and intrinsic reward functions. Refer to Appendix~\ref{appendix:prompts} for details on all prompt templates. (2) $K$ state representations and intrinsic rewards $\{\mathcal{F}_k\}_{k=1}^K, \{\mathcal{G}_k\}_{k=1}^K$ are sampled from LLM. (3) During RL training, function $\mathcal{F}$ and $\mathcal{G}$ are utilized to generate $s^r = \mathcal{F}(s)$ for state representations, and $r^i = \mathcal{G}(s, s^r)$ for intrinsic rewards. (4) Finally, Lipschitz constants and episode returns of each candidate serve as feedback metrics for LLM.}
	\label{fig2:our_framework}
\end{figure*}

\textbf{State Representation Derived from LLM} \quad 
Some researchers use LLM for state representations that contain more information in partially observable scenarios. \citet{da2023llm} employs LLMs to provide extra details about road conditions in traffic control areas. Similarly, \citet{shek2023lancar} focuses on robot locomotion in unstructured environments, using LLM as translators to extract environmental properties and generate contextual embeddings for training. \citet{chen2023llmstate}, tracks key objects and attributes in open-world household environments via LLM, expanding and updating object attributes based on historical trajectory information. These methods mainly addressed the issue of missing state information in partially observable scenarios. In contrast, our method emphasizes exploring correlations among internal state features, recognizing meaningful correlations reflecting underlying physical relationships and generating more task-related representations. 

\textbf{Reward Design via LLM} \quad 
For the purpose of effectively bridging the gap between high-level language instructions and low-level robot actions, some researchers employ rewards as the intermediate interface generated by LLM. Works in this domain can be categorized into the following three types: (1) Sparse Reward: \citet{kwon2023reward, yu2023language, sontakke2023roboclip}, aiming to design sparse rewards at the trajectory level. (2) Dense Reward: \citet{song2023self, xie2023text2reward, ma2023eureka, rocamonde2023visionlanguage}, aiming to design dense rewards for every interactive step of the agent. (3) Intrinsic Reward: \citet{klissarov2023motif, triantafyllidis2023intrinsic}, aiming to design intrinsic rewards for reducing ineffective exploration and improving sample efficiency. In this paper, we also utilize LLM to generate intrinsic reward function codes. The primary difference between our methodology and prior research lies in that our reward design serves as an auxiliary mechanism for encouraging the agent to better comprehend the state representations generated by LLM. This aids the policy and critic networks in establishing correlations between the state representations and intrinsic rewards.

\section{Method}

\subsection{Problem Statement}

We consider a Markov decision process~\cite{puterman1990markov} defined by a tuple ($\mathcal{S}, \mathcal{A}, \mathcal{R}, \mathcal{P}, p_0, \gamma$), where $\mathcal{S}$ denotes the source state space and $\mathcal{A}$ denotes the action space.
Given a specific task, $\mathcal{R}$ is the source extrinsic reward function of the environment. $\mathcal{P}(s'|s, a)$ denotes the dynamic transition function, $p_0$ is the initial state distribution, and $\gamma$ is the discount factor.
The primary objective is to learn a RL policy $\pi(a|s)$ that maximizes the cumulative rewards expectation, which is defined as value function $Q_\pi(s_t, a_t) = \mathbb{E}_\pi \Big[ \sum_{t=0}^{\infty} \gamma^t r_t \Big|s_t, a_t  \Big]$.

In order to assess the impact of LESR on network continuity, we introduce the Lipschitz constant~\cite{jones1993lipschitzian}:
\begin{definition}\label{defi:constant}
	Denote the data space as $\mathcal{X} \subset \mathbb{R}^d$ and the label space as $\mathcal{Y} \subset \mathbb{R}$.
	Consider a dataset $\mathcal{X}_0 \subset \mathcal{X}$, and the label $\mathcal{Y}_0 = \{y_i | y_i=u(x_i), \text{where } x_i \in \mathcal{X}_0\} \subset \mathcal{Y}$.
	Here, $x_i$ represents a sequence of i.i.d. random variables on $\mathcal{X}$ sampled from the probability distribution $\rho$, and $u:\mathcal{X}_0 \subset \mathcal{X}\rightarrow \mathcal{Y}$ is the Lipschitz constant of a mapping given by
	\begin{align}
		\Lp(u;\mathcal{X}_0) = \sup_{x_1,x_2 \in \mathcal{X}_0}   \frac{\|u(x_1)-u(x_2)\|_2}{\|x_1 - x_2\|_2}.
	\end{align}
	When $\mathcal{X}_0$ is all of $\mathcal{X}$, we write $\Lp(u;\mathcal{X}) = \Lp(u)$.
	A lower Lipschitz constant indicates a smoother mapping $u$.
\end{definition}

\subsection{LLM-Empowered State Representation}\label{sec:state_feature_expansion}

In many RL settings~\cite{todorov2012mujoco, brockman2016openai, de2023gymnasium}, source state representations usually contain general environmental information, while often lacking specific details related to current tasks which is critical to the training of the value networks~\cite{yang2020multi, yoo2022skills}.
The absence of task-related representations may hinder network mappings from states to rewards, impacting the continuity of networks.
Recognizing this limitation, the identification and incorporation of additional task-related state representations emerge as a pivotal strategy. This strategic augmentation can expedite the establishment of network mappings, subsequently boosting the smoothness of networks and augmenting training efficiency.

Due to the extensive knowledge and priors embedded in LLM, utilizing it for generating task-related state representations can be promising.
In this context, we present a direct and efficient method named \textbf{L}LM-\textbf{E}mpowered \textbf{S}tate \textbf{R}epresentation (\textbf{LESR}).
The whole framework is depicted in Figure~\ref{fig2:our_framework}. 
Our methodology hinges on leveraging LLM to facilitate the generation of more task-specific state representations.
Herein, we denote LLM as $\mathcal{M}$ and a descriptor translating symbolic information into natural language as $d$. Consequently, the input to LLM $\mathcal{M}$ is represented as $d(\mathcal{S})$, which constitutes four parts: (1) \textit{Task Description}: information about the RL environment and the specific task. (2) \textit{State Details}: information pertaining to each dimension of the source state. (3) \textit{Role Instruction}: assignments that require LLM to generate task-related state representation and intrinsic reward codes. (4) \textit{Feedback}: historical information from previous iterations.
For a full comprehensive descriptions of our prompts, refer to Appendix~\ref{appendix:prompts}. 

Our primary objective is to harness LLM $\mathcal{M}$ to formulate a python function $\mathcal{F}: \mathcal{S}\rightarrow\mathcal{S}^r$, where $\mathcal{S}^r$ denotes the LLM-empowered state representation space. $\mathcal{F}$ is sampled from $\mathcal{M}\Big( d(\mathcal{S}) \Big)$ and $d(\mathcal{S})$ explicitly embeds the task information into LLM. 
The state representation function $\mathcal{F}$ utilizes source state dimensions for calculations, generating task-specific state dimensions.
At each timestep $t$, when the agent get the current state $s_t$ from the environment, the corresponding state representation $s^r_t = \mathcal{F}(s_t)$ will be concatenated to the source state as the input for the policy $\pi(a_t | s_t, s^r_t)$ and value $Q(s_t, s^r_t, a_t)$.  

Once the state representation function $\mathcal{F}$ is obtained, LLM $\mathcal{M}$ is subsequently required to provide an intrinsic reward function $\mathcal{G}: \mathcal{S}^c\rightarrow\mathbb{R}$ in python code format based on $s_t^c = (s_t, s_t^r)\in \mathcal{S}^c$, where $\mathcal{S}^c = \mathcal{S} \times \mathcal{S}^r$ is the joint state space.
More precisely, we stipulate in the prompt that LLM is obliged to incorporate the LLM-empowered state representations $s_t^r$ to calculate the intrinsic rewards, and it also retains the option to incorporate the source state $s_t$ for a better intrinsic reward design.
We formulate the joint optimization objective as:
\begin{equation}
	\begin{aligned}
		\max\limits_{\mathcal{F}, \mathcal{G}}\max\limits_{\pi} \space &\mathbb{E}_{\mathcal{F}, \mathcal{G}, \pi} \Big[ \sum_{t=0}^{\infty} \gamma^t \Big( r + w \cdot r^i\Big) \Big|
		r^i = \mathcal{G}\Big(s_t, \mathcal{F}(s_t)\Big)  \Big].
	\end{aligned} 
	\label{eq:max_obj_2}
\end{equation} 
where $\mathcal{F}, \mathcal{G} \sim\mathcal{M}\Big(d\left(\mathcal{S}\right)\Big)$, and $w$ is the weight of the intrinsic reward.

\subsection{Lipschitz Constant for Feedback} \label{sec:lip_feedback}

In practice, to enhance the robustness of state representations, we iteratively query LLM multiple times, incorporating previous training results as feedback.
During each training iteration, we sample $K$ state representation and intrinsic reward function codes $\mathcal{F}_k, \mathcal{G}_k, k=1, \dots, K$ from LLM $\mathcal{M}$.
Subsequently, we concurrently execute $K$ training processes over $N_{small}$ training timesteps to evaluate the performance of each function $\mathcal{F}_k, \mathcal{G}_k$. Here, $N_{small}$ is intentionally set to be smaller than the total timesteps $N$ employed in the final evaluation stage. 

\textbf{$\bullet$ Continuous Scenarios} For scenarios with continuous extrinsic reward, we maintain a Lipschitz constant array $C_k\in \mathbb{R}^{|\mathcal{S}^c|}$ for each of the $K$ training instances.
Each element of $C_k$ signifies the Lipschitz constant of a mapping $u_i, i=1, \dots, |\mathcal{S}^c|$, which maps each dimension of $\mathcal{S}^c$ to the extrinsic rewards. Note: $u_i$ is introduced to signify the Lipschitz constant computed independently for each state dimension concerning the extrinsic reward. This assessment is crucial for guiding the LLM in identifying and eliminating undesired dimensions within the state representation. 
Given a trajectory $T=\{s_t^c, r_t\}_{t=1}^{H}$ of length $H$, the current Lipschitz constant array is calculated as follows: 
\begin{equation}
	\begin{aligned}
		C_k^{T} = \Bigg[ \Lp(u_i;T_i) \Bigg]_{i=1}^{|\mathcal{S}^c|},
	\end{aligned} 
	\label{eq:corr_1}
\end{equation}
where $T_i = \{s^c_t[i], r_t\}_{t=1}^H$, $s^c_t[i]$ denotes the $i$-th dimension of the joint state representations $s^c$, and $C_k^{T}$ denotes the Lipschitz constant array of the current trajectory. we soft-update $C_k$ over trajectories:
\begin{equation}
	\begin{aligned}
		C_k = \tau C_k + (1 - \tau) C_k^{T},
	\end{aligned} 
	\label{eq:corr_2}
\end{equation}
where $\tau \in [0, 1]$ is the soft-update weight.
At the end of each training iteration, the Lipschitz constant array $C_k$ and policy performance $\nu_k$ are provided to LLM for CoT~\cite{wei2022chain} suggestions, which, along with the training results, serve as feedback for subsequent iterations.
The feedback information helps LLM to generate task-related state representations that exhibit a lower Lipschitz constant with extrinsic rewards, as elaborated in Section~\ref{sec:theory} where we discuss the theoretical advantages of a lower Lipschitz constant for network convergence.
In the subsequent iterations, LLM leverages all historical feedback to iteratively refine and generate improved state representation and intrinsic reward function codes $\{\mathcal{F}_k\}_{k=1}^K, \{\mathcal{G}_k\}_{k=1}^K$. 
The whole algorithm is summarized in Algorithm~\ref{alg:ours}.  

\begin{algorithm}[t]
	\caption{LLM-Empowered State Representation}\label{alg:ours}
	\footnotesize
	\vspace{0.5em}
	\textbf{Input}:  Policy network $\pi_\theta$, value network $Q_\phi$, LLM $\mathcal{M}$, training timesteps $N_{small}$, total final evaluation timesteps $N$, sample count $K$, iteration count $I$, initial prompt $l_1 = d(\mathcal{S})$, feedback analysis prompt $l_{analysis}$.\\
	\textbf{Output}: learned policy network $\pi_\theta$, best state representation and intrinsic reward function $\mathcal{F}_{best}, \mathcal{G}_{best}$.\vspace{-0.5em}\\
	\begin{algorithmic}
        \STATE \texttt{\# Sampling and Training Stage}
		\FOR {itr = 1 to $I$}
		\STATE Sample $\mathcal{F}_k, \mathcal{G}_k \sim \mathcal{M}(l_{\text{itr}}), k=1,\dots,K$
		\STATE Concurrently execute $K$ training processes over $N_{small}$ training timesteps using $\mathcal{F}_k, \mathcal{G}_k$.
		\STATE Obtain $C_k$ according to Equation~\eqref{eq:corr_1}, \eqref{eq:corr_2} and final policy performance $\nu_k$.
		\STATE $l_{feedback} = \mathcal{M}(l_{analysis};\{C_k\};\{\nu_k\})$
		\STATE $l_{\text{itr}+1} = \Big(l_{\text{itr}};\space l_{feedback}\Big)$
		\ENDFOR	
		\STATE \texttt{\# Evaluating Stage}
		\STATE Initialize $\pi_\theta$ and $Q_\phi$.
		\STATE Execute $N$ training timesteps using $\mathcal{F}_{best}, \mathcal{G}_{best}$.
		\STATE \textbf{return} $\pi_\theta$, $\mathcal{F}_{best}, \mathcal{G}_{best}$.
	\end{algorithmic}
\end{algorithm} 

\textbf{$\bullet$ Discontinuous Scenarios} Dealing with scenarios with discontinuous extrinsic reward using conventional RL baselines is notably challenging and constitutes a specialized research area~\cite{vecerik2017leveraging, trott2019keeping, liu2023lazy}. Despite this challenge, LESR remains effective in such scenarios. We provide two ways of estimating Lipschitz constant as feedback in discontinuous extrinsic reward settings.

\textbf{LESR with Discounted Return} In Equation~\ref{eq:corr_1}, $u_i$ initially maps each dimension of $\mathcal{S}^c$ to dense extrinsic rewards. In sparse reward settings, we substitute these extrinsic rewards with the discounted episode return $\sum_{t}\gamma^t r$. Thus, $u_i$ now maps each dimension of $\mathcal{S}^c$ to the discounted episode returns, consistent with the algorithmic framework and theoretical scope proposed in LESR. 

\textbf{LESR with Spectral Norm} Since in Theorems~\ref{theo:value_lip} and \ref{theo:value_lip_convergence} show that reducing the Lipschitz constant of the reward function lowers the upper bound of $Lip(V; \mathcal{S})$ and improves the convergence of value functions. Therefore, we can use the spectral norm to estimate $Lip(V; \mathcal{S})$~\cite{anil2019sorting, fazlyab2019efficient} as feedback to LLM. By calculating the spectral norm of the $N$ weight matrices $W_1, \dots, W_N$ of the value functions, the Lipschitz constant of the value function is bounded by $\prod_{i=1}^{N} \|W_i\|_2$, which is then presented to the LLM as feedback.

\subsection{Theoretical Analysis}\label{sec:theory}
In this section, we present analysis of the theoretical implications of the Lipschitz constant on convergence in neural networks, inspired by \citet{oberman2018lipschitz}.
Consider the dataset and labels $\mathcal{X}_0, \mathcal{Y}_0$ defined in Definition~\ref{defi:constant} with $N=|\mathcal{X}_0|$ elements.
The true mapping from $\mathcal{X}$ to $\mathcal{Y}$ is denoted as $u_0^*: \mathcal{X}\rightarrow \mathcal{Y}$.
Let $f:\mathcal{X} \rightarrow \mathcal{X}$ be a function transforming a source $x\in \mathcal{X}$ into a more task-related $f(x)$.

\begin{definition}\label{eq:emprical_loss}
	Denote $u(x;\psi)$ as a neural network mapping parameterized by $\psi$. Consider the empirical loss for $\mathcal{X}_0, \mathcal{Y}_0$, where $\ell : \mathcal{Y}\times \mathcal{Y} \to \mathbb{R}$ is a loss function satisfying  (i)  $\ell \geq 0$, (ii)  $\ell(y_1,y_2) = 0$ if and only if $y_1 = y_2$:
	\begin{equation}\label{eq:loss}
		\min_{u: \mathcal{X}_0 \to \mathcal{Y}}  \mathcal{L}(u, \mathcal{X}_0) =  \frac{1}{N} \sum_{i=1}^N\ell(u(x_i; \psi), y_i).
	\end{equation}
\end{definition}

\begin{assumption}\label{assumption}
	The mapping $f:\mathcal{X} \rightarrow \mathcal{X}$ only swaps the order of $x$ in the dataset $\mathcal{X}_0$, which means $\mathcal{X}_1 = \{f(x_i) | f(x_i) \in \mathcal{X}_0, i =1,\dots,N\}$ and $\mathcal{X}_1 = \mathcal{X}_0$.
	While for $x\in\mathcal{X}_2=\{x| x\in\mathcal{X}, x \notin \mathcal{X}_0\}$, $f(x) = x$. 
    Under $f$, the true mapping from $f(\mathcal{X})$ to $\mathcal{Y}$ is denoted as $u_1^*: f(\mathcal{X})\rightarrow \mathcal{Y}$. It can be derived that $u_1^* = u_0^* \circ f^{-1}$.
    We suppose under $f$ a lower Lipschitz constant is achieved:
	\begin{equation} 
		\begin{aligned}
			\Lp(u_1^*) &\leq \Lp(u_0^*).
		\end{aligned}
	\end{equation}
\end{assumption}

\begin{theorem}\label{theorem:convergence}
	Under Assumption~\ref{assumption}, Given $\mathcal{X}_0, \mathcal{Y}_0$ and $\mathcal{X}_1, \mathcal{Y}_1=\{y_i | y_i=u_1^*(x_i), x_i \in \mathcal{X}_1\}$, $u_0\in\mathcal{U}_0$ is any minimizer of $\mathcal{L}(u, \mathcal{X}_0)$ and $u_1\in\mathcal{U}_1$ is any minimizer of $\mathcal{L}(u, \mathcal{X}_1)$, where $\mathcal{U}_0$ and $\mathcal{U}_1$ denote the solution set of $\mathcal{L}(u, \mathcal{X}_0)$ and $\mathcal{L}(u, \mathcal{X}_1)$ in Definition~\ref{eq:emprical_loss} relatively, then on the same condition when $\Lp(u_0) = \Lp(u_1)$:
	\begin{equation}\label{eq:lower_bound}
		\sup_{u_1\in\mathcal{U}_1} \mathbb{E}_{x \sim \rho_f} \|u_1^* - u_1\|_2 \leq \sup_{u_0\in\mathcal{U}_0} \mathbb{E}_{x \sim \rho}\|u_0^* - u_0\|_2,
	\end{equation}
\end{theorem}

$\rho$ and $\rho_f$ denote the source probability distribution on $\mathcal{X}$ and probability distribution on $f(\mathcal{X})$, relatively. In Theorem~\ref{theorem:convergence}, it is demonstrated that the mapping $f$ exhibiting a lower Lipschitz constant can attain superior convergence. This observation underscores the significance of identifying task-related state representations characterized by lower Lipschitz constants with respect to the associated rewards.
Such analysis to some extent sheds light on why smoother network mappings exhibit improved convergence performance.
Proofs of Theorem~\ref{theorem:convergence} can be referred to Appendix~\ref{appendix:proof}.

We delve deeper into the significance of the Lipschitz constant of the reward concerning state representations in RL. We introduce two additional theorems, namely Theorem~\ref{theo:value_lip} and Theorem~\ref{theo:value_lip_convergence}, establishing a strong correlation between $Lip(r; \mathcal{S})$ and $Lip(V; \mathcal{S})$ and, consequently, the convergence of RL's value functions. Theorem~\ref{theo:value_lip} indicates that reducing the Lipschitz constant of the reward function lowers the upper bound of $Lip(V; \mathcal{S})$. Theorem~\ref{theo:value_lip_convergence} illustrates how decreasing $Lip(r; \mathcal{S})$ can enhance the convergence of RL algorithms' value functions. These theorems collectively emphasize our focus on minimizing the Lipschitz constant of the reward function to improve RL algorithms' convergence. Detailed proofs are available in Appendix~\ref{appen:why_crucial}.

\section{Experiments}
\label{experiments}

In this section, we will assess \textbf{L}LM-\textbf{E}mpowered \textbf{S}tate \textbf{R}epresentation (\textbf{LESR}) through experiments on two well-established reinforcement learning (RL) benchmarks: Mujoco~\cite{todorov2012mujoco, brockman2016openai} and Gym-Robotics~\cite{de2023gymnasium}. For more information about the tasks, see Appendix~\ref{appendix:env_info}.  The following questions will guide our investigation:

\textbullet \textbf{Q1:} Can LESR generate task-related state representations characterized by lower Lipschitz constants in relation to extrinsic environmental rewards?
(Section~\ref{sec:demo_tsne})

\textbullet \textbf{Q2:} Can LESR achieve higher sample efficiency and outperform RL baselines? Does each component contribute to the final performance? (Section~\ref{sec:performance},~\ref{sec:ablation})

\textbullet \textbf{Q3:} Are functions $\mathcal{F}_{best}$ and $\mathcal{G}_{best}$ algorithm-agnostic, and transferable directly to other RL algorithms? (Section~\ref{sec:algorithm_transfer})

\textbullet \textbf{Q4:} Do $\mathcal{F}_{best}$ and $\mathcal{G}_{best}$ possess semantic-physical significance and exhibit consistency across different runs of LLM? Does LESR exhibit robustness to the variations in hyperparameters? (Section~\ref{sec:semantic_analysis},~\ref{sec:robostness})

\subsection{Implementation Details}

\textbf{LLM and Prompts} We employ the gpt-4-1106-preview as LLM to generate the state representation and intrinsic reward functions.
There are three well-designed prompt templates and details of prompts are available in Appendix~\ref{appendix:prompts}.  

\textbf{Baseline Algorithm} We employ the SOTA RL algorithm TD3~\cite{fujimoto2018addressing} as the foundational Deep Reinforcement Learning~(DRL) algorithm. Building upon the implementation of TD3 as provided in the source paper\footnote{https://github.com/sfujim/TD3}, we have formulated LESR~(Ours). We also employ EUREKA~\cite{ma2023eureka} for comparison, which incorporates LLM for human-level reward designs. It is noteworthy that, in order to maintain comparative fairness, we have adhered to the hyperparameters of TD3 without introducing any modifications. Both EUREKA and LESR are grounded in the common RL algorithm TD3. For a comprehensive list of hyperparameters, please refer to Appendix~\ref{appendix:hypers}. 

\subsection{LESR Can Enhance Lipschitz Continuity}\label{sec:demo_tsne}

The reward function $\mathcal{R}$ commonly serves as an indicator of the specific task within the same environment~\cite{ yoo2022skills}. In RL, value function learning is also predicated on rewards. Specifically, when the discount factor $\gamma$ equals to 0, the value function directly learns the rewards. Additionally, the state representations form a part of the input to the value function. Task-related state representations might contribute to enhancing the Lipschitz continuity of the value function, thereby expediting the learning process.

Therefore, to validate whether LESR can help enhance the Lipschitz continuity between the generated state representations and the value function, we execute the final policy of LESR in the Mujoco \texttt{Ant} environment for 20 episodes of length $H=1000$ and establish two datasets: (a) $T_1 = \{s_t, r_t\}_{t=1}^H$ and (b) $T_2 = \{\mathcal{F}_{best}(s_t), r_t\}_{t=1}^H$. For states within each dataset, we employ t-SNE~\citep{van2008visualizing} to visualize the states on the 2D graph with coloring the data with corresponding rewards and calculate the Lipschitz constant between states and rewards.

As shown in Figure~\ref{fig:demo_lipschitz}, it is illustrated that in $T_1$ the Lipschitz constant of the mapping from state representations generated by LESR to the extrinsic environment rewards ($\mathcal{R}:\mathcal{F}(s)\to r, \Lp(\mathcal{R};T_2) = 168.1$) is much lower than that of the source($\mathcal{R}:s\to r, \Lp(\mathcal{R};T_1) = 560.2$). This indicates that the task-related state representations generated by LESR can enhance the Lipschitz continuity.

\begin{figure}[t]
	\centering
	
	\begin{subfigure}{0.45\linewidth}
		\centering
		\includegraphics[width=\linewidth]{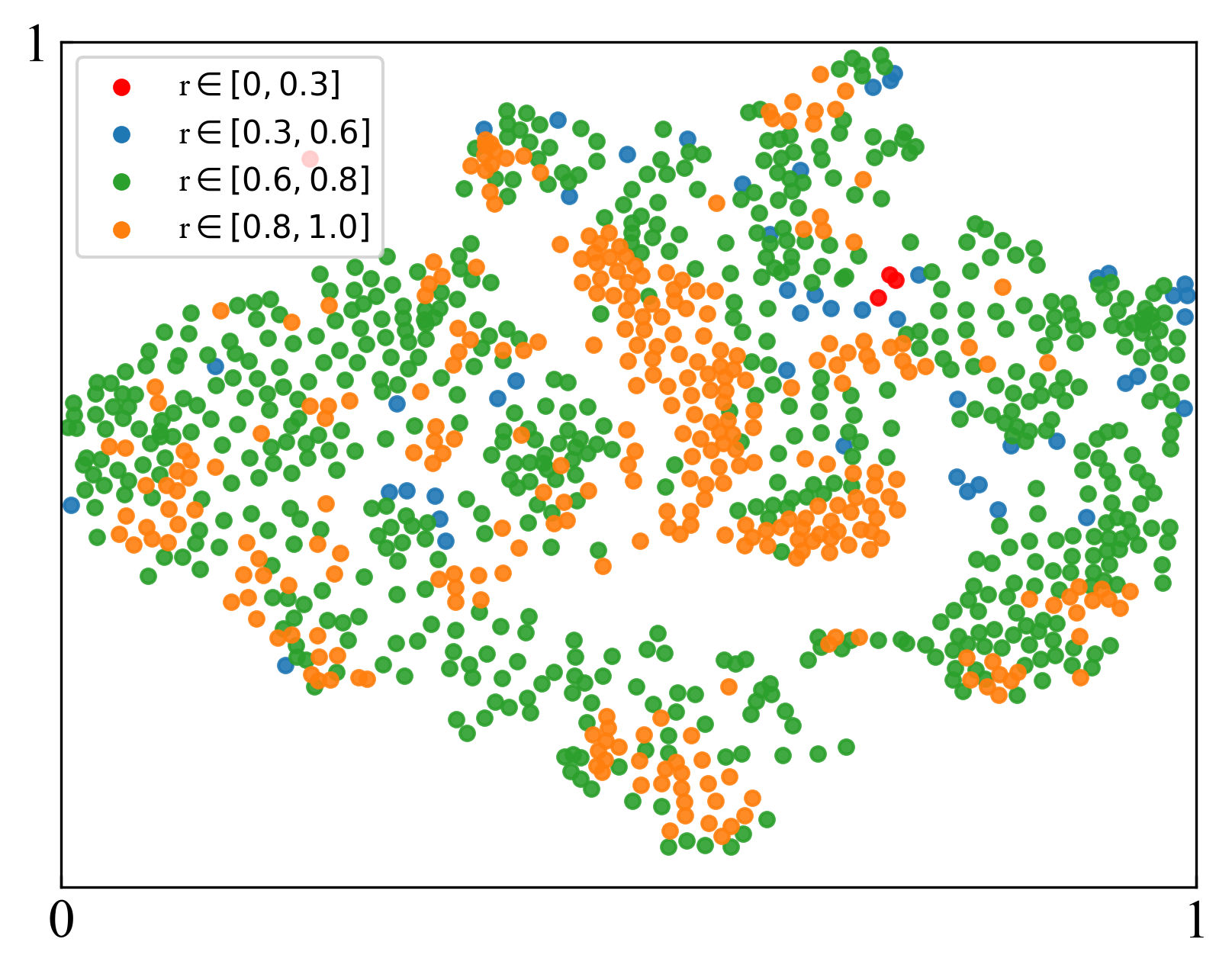}
		\captionsetup{font=scriptsize}
		\caption{$\Lp(\mathcal{R};T_1) = 560.2$}
	\end{subfigure}
	\begin{subfigure}{0.45\linewidth}
		\centering
		\includegraphics[width=\linewidth]{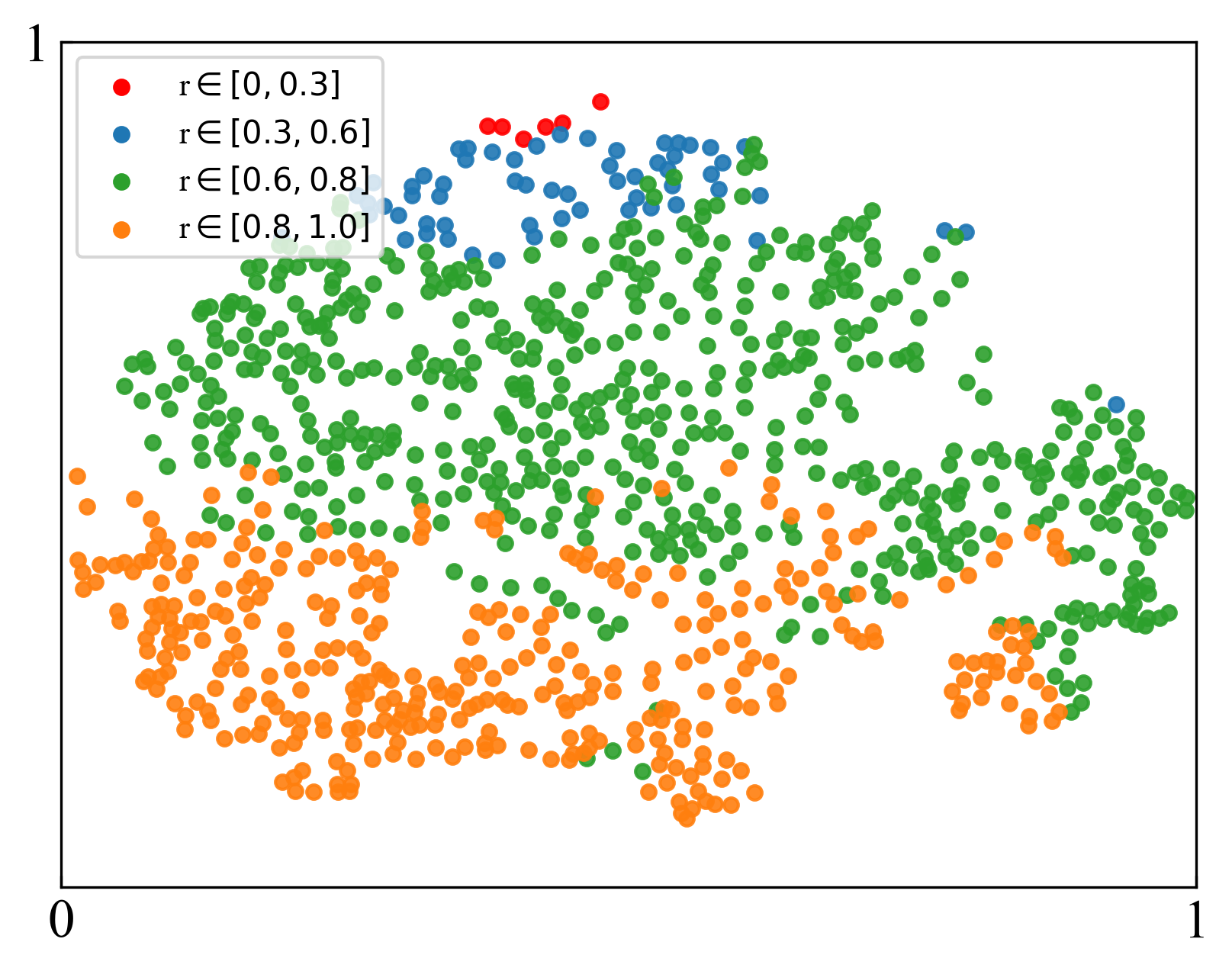}
		\captionsetup{font=scriptsize}
		\caption{$\Lp(\mathcal{R};T_2) = 168.1$}
	\end{subfigure}
	
	\caption{Visualization illustrating states post 2D dimensionality reduction via t-SNE. Details of $T_1$ and $T_2$ can be referred to Section~\ref{sec:demo_tsne}. The reward for each state is normalized to a range of $[0, 1]$ and discretized, and the graph employs color coding to represent their respective reward values.}
	\vspace{-10pt}
	\label{fig:demo_lipschitz}
\end{figure}

\subsection{LESR Can Achieve High Sample Efficiency} \label{sec:performance} 

\textbf{Performance Comparison} We validate LESR and the results are presented in Table~\ref{tab1:all_run}. In Mujoco environments, LESR outperforms the state-of-art baselines in 4 out of 5 tasks. In Gym-Robotics environments, it outperforms the baselines in 5 out of 6 tasks. Particularly noteworthy are the results on challenging antmaze tasks, where TD3 fails to train, resulting in all final success rates of TD3 remaining at zero. Conversely, the performance exhibited by LESR is marked by excellence, highlighting its proficiency in overcoming the challenges posed by these intricate tasks. LESR yields an average performance improvement of 29\% over the 5 tasks in Mujoco and an average success improvement of 30\% over the 7 tasks in Gym-Robotics, thus substantiating the efficacy of the employed methodology. More information of training can be referenced in Appendix~\ref{appendix:training_curve}.

\begin{table}[t]
	\centering
	\caption{Final Performance on Mujoco environments over only \textbf{300k} environment interaction training steps, aiming to validate the sample efficiency of LESR. \textbf{RPI} = \texttt{(accumulated\_rewards - baseline)/baseline}: relative performance improvement.}
	\label{table:300k}
	\resizebox{.95\columnwidth}{!}{
		\begin{tabular}{c|c|cc}
			\toprule
			\diagbox{\textbf{Environments}}{\textbf{Algorithm}} & \textbf{TD3} & \textbf{LESR(Ours)} & \textbf{RPI} \\
			\midrule
			HalfCheetah & 7288.3$\pm$842.4 & \textbf{7639.4$\pm$464.1} & \textbf{5\%} \\
			Hopper & 921.9$\pm$646.8 & \textbf{2705.0$\pm$623.2} & \textbf{193\%} \\
			Walker2d & 1354.5$\pm$734.1 & \textbf{1874.8$\pm$718.2} & \textbf{38\%} \\
			Ant & 1665.2$\pm$895.5 & \textbf{1915.3$\pm$885.5} & \textbf{15\%} \\
			Swimmer & 49.9$\pm$5.2 & \textbf{150.9$\pm$9.5} & \textbf{203\%} \\
			\textbf{Mujoco Improve Mean} & - &  - & \textbf{91\%} \\
			\bottomrule
	\end{tabular} }
\end{table}

\textbf{Sample Efficiency} For the purpose of assessing the sample efficiency of LESR, we have presented the performance for a limited scope of 300k training steps, as depicted in Table~\ref{table:300k}. The results demonstrate that LESR exhibits superior performance within significantly fewer training steps, excelling comprehensively across all tasks in comparison to the SOTA baseline, thereby manifesting heightened sample efficiency. This is further supported by the information of experimental details 
presented in Appendix~\ref{appendix:training_curve}. 

\begin{table*}[htbp]
	\centering
	\caption{Final Performance on Mujoco and Gym-robotics environments over \textbf{1M} environment interaction training steps. \textbf{Ours w/o IR:} without intrinsic reward. \textbf{Ours w/o SR:} without state representation. \textbf{Ours w/o FB:} without feedback. More details about the ablation study are elucidated in Section~\ref{sec:ablation}. \textbf{RPI} = \texttt{(accumulated\_rewards - baseline)/baseline}: relative performance improvement. \textbf{PI} = \texttt{success\_rate - baseline}: success rate improvement. This distinction arises because the performance is represented by the accumulative rewards in Mujoco, whereas in Gym-Robotics, it is measured by the success rate. The "mean$\pm$std" denotes values computed across five random seeds.}
	\label{tab1:all_run}
	\resizebox{\textwidth}{!}{
		\begin{tabular}{c|c|cc|cc|cc|cc|cc}
			\toprule
			\diagbox{\textbf{Environments}}{\textbf{Algorithm}} & \textbf{TD3} & \textbf{EUREKA} & \textbf{RPI} & \textbf{Ours w/o IR} & \textbf{RPI} & \textbf{Ours w/o SR} & \textbf{RPI} & \textbf{Ours w/o FB} & \textbf{RPI} &\textbf{LESR(Ours)} & \textbf{RPI} \\
			\midrule
			HalfCheetah & 9680.2$\pm$1555.8 & 10400.6$\pm$289.2 & 7\% & 9969.8$\pm$1767.5 & 3\% & 9463.2$\pm$796.3 & -2\% & 9770.4$\pm$1531.3 & 1\% & \textbf{10614.2$\pm$510.8} & \textbf{10\%} \\
			Hopper & 3193.9$\pm$507.8 & 3346.4$\pm$423.2 & 5\% & 3324.7$\pm$191.7 & 4\% & 3159.0$\pm$466.7 & -1\% & 2851.3$\pm$748.1 & -11\% & \textbf{3424.8$\pm$143.7} & \textbf{7\%} \\
			Walker2d & 3952.1$\pm$445.7 & 3606.2$\pm$1010.0 & -9\% & 4148.2$\pm$352.9 & 5\% & 3977.4$\pm$434.3 & 1\% & 4204.9$\pm$590.3 & 6\% & \textbf{4433.0$\pm$435.3} & \textbf{12\%} \\
			Ant & 3532.6$\pm$1265.3 & 2577.7$\pm$1085.6 & -27\% & \textbf{5359.0$\pm$336.4} & \textbf{52\%} & 3962.2$\pm$1332.0 & 12\% & 4244.9$\pm$1227.9 & 20\% & 4343.4$\pm$1171.4 & 23\% \\
			Swimmer & 84.9$\pm$34.0 & 98.1$\pm$31.1 & 16\% & 132.0$\pm$6.0 & 55\% & 160.2$\pm$10.2 & 89\% & 85.4$\pm$29.8 & 1\% & \textbf{164.2$\pm$7.6} & \textbf{93\%} \\
			\textbf{Mujoco Relative Improve Mean} & - & - & -2\% &  - & 24\% &  - & 21\% &  - & 3\% &  - & \textbf{29\%} \\
			\midrule
			- & - & - & \textbf{PI} & - & \textbf{PI} &  - & \textbf{PI} &  - & \textbf{PI} &  - & \textbf{PI} \\
			\midrule
			AntMaze\_Open & 0.0$\pm$0.0 & 0.13$\pm$0.04 & 13\% & 0.0$\pm$0.0 & 0\% & 0.15$\pm$0.06 & 15\% & 0.16$\pm$0.07 & 16\% & \textbf{0.17$\pm$0.06} & \textbf{17\%} \\
			AntMaze\_Medium & 0.0$\pm$0.0 & 0.01$\pm$0.01 & 1\% & 0.0$\pm$0.0 & 0\% & 0.07$\pm$0.05 & 7\% & 0.0$\pm$0.0 & 0\% & \textbf{0.1$\pm$0.05} & \textbf{10\%} \\
			AntMaze\_Large & 0.0$\pm$0.0 & 0.0$\pm$0.0 & 0\% & 0.0$\pm$0.0 & 0\% & 0.06$\pm$0.03 & 6\% & 0.04$\pm$0.04 & 4\% & 0.07$\pm$0.03 & 7\% \\
			FetchPush & 0.07$\pm$0.02 & 0.07$\pm$0.03 & 0\% & 0.78$\pm$0.16 & 71\% & 0.06$\pm$0.05 & -1\% & 0.77$\pm$0.14 & 70\% & \textbf{0.91$\pm$0.08} & \textbf{84\%} \\
			AdroitHandDoor & 0.53$\pm$0.44 & 0.83$\pm$0.08 & 30\% & 0.46$\pm$0.46 & -7\% & 0.63$\pm$0.39 & 10\% & 0.23$\pm$0.38 & -30\% & \textbf{0.88$\pm$0.12} & \textbf{34\%} \\
			AdroitHandHammer & 0.28$\pm$0.32 & 0.32$\pm$0.45 & 4\% & 0.22$\pm$0.21 & -6\% & 0.29$\pm$0.28 & 1\% & 0.41$\pm$0.33 & 13\% & \textbf{0.53$\pm$0.38} & \textbf{25\%} \\
			\textbf{Gym-Robotics Improve Mean} & - & - & 8\% &  - & 8\% &  - & 6\% &  - & 10\% &  - & \textbf{30\%} \\
			\bottomrule
		\end{tabular} 
	}
\end{table*}

\subsection{Ablation Study}\label{sec:ablation}

\textbf{Experimental Setting} In this section, we conduct ablative analyses on distinct components of LESR to evaluate their respective contributions. Three ablation types are considered in total. \textbf{Ours w/o IR:} This entails the removal of the intrinsic reward component, with training relying solely on state representation. \textbf{Ours w/o SR:} Here, we exclude the state representation part. It is crucial to note that despite this ablation, the state representation function is still necessary due to the intrinsic reward calculation process $r^i = \mathcal{G}\Big(s, \mathcal{F}(s)\Big)$ outlined in Equation~\eqref{eq:max_obj_2}. However, in this context, only the source state $s$ is provided as input to the policy, instead of the concatenated state $s^c$. \textbf{Ours w/o FB:} In this instance, the `iteration\_count' specified in Appendix~\ref{appendix:hypers} is set to 1 and the `sample\_count' is increased to 10, eliminating the feedback component.  

\begin{table*}[t]
	\begin{minipage}[b]{0.56\linewidth}
		\centering
		\caption{Gym-Robotics Results. \textbf{PPO} and \textbf{SAC} denote algorithms trained using the original state and reward functions within the tasks. \textbf{+ Ours} signifies training utilizing state representation functions and the corresponding intrinsic reward functions. \textbf{PI}: performance improvement.}
		\label{tab:algotransfer}
		\resizebox{\textwidth}{!}{
			\begin{tabular}{c|ccc|ccc}
				\toprule
				\diagbox{\textbf{Environments}}{\textbf{Algorithm}} & \textbf{PPO} & \textbf{PPO + Ours} & \textbf{PI} & \textbf{SAC} & \textbf{SAC + Ours} & \textbf{PI} \\
				\midrule
				AntMaze\_Open & 0.006$\pm$0.005 & \textbf{0.12$\pm$0.022} & \textbf{11\%} & 0.002$\pm$0.001 & \textbf{0.12$\pm$0.033} & \textbf{12\%} \\
				AntMaze\_Medium & 0.0$\pm$0.0 & \textbf{0.156$\pm$0.039} & \textbf{16\%} & 0.008$\pm$0.003 & \textbf{0.108$\pm$0.03} & \textbf{10\%} \\
				AntMaze\_Large & 0.006$\pm$0.004 & 0.092$\pm$0.02 & 9\% & 0.009$\pm$0.003 & \textbf{0.124$\pm$0.029} & \textbf{12\%} \\
				FetchPush & 0.063$\pm$0.017 & \textbf{0.812$\pm$0.13} & \textbf{75\%} & 0.059$\pm$0.015 & \textbf{0.589$\pm$0.207} & \textbf{53\%} \\
				AdroitHandDoor & 0.001$\pm$0.001 & \textbf{0.948$\pm$0.047} & \textbf{95\%} & 0.759$\pm$0.158 & \textbf{0.956$\pm$0.035} & \textbf{20\%} \\
				AdroitHandHammer & 0.006$\pm$0.002 & 0.007$\pm$0.002 & 0\% & 0.752$\pm$0.158 & \textbf{0.892$\pm$0.084} & \textbf{14\%} \\
				\textbf{Gym-Robotics Improve Mean} & - & - & \textbf{34\%} & - & - & \textbf{20\%}\\
				\bottomrule
		\end{tabular} }
	\end{minipage}
	\hspace{5pt}
	\begin{minipage}[b]{0.42\linewidth}
        \centering
          \caption{Consistency verification in \texttt{Swimmer}. $seed_i$ denotes the experiments over four random seeds. $c_i$ signifies the state representation categories. \checkmark signifies the inclusion of a particular category in the corresponding seed.}
        \label{tab:consis}%
        \resizebox{.8\textwidth}{.18\textwidth}{
            \begin{tabular}{c|cccc}
            \toprule
                  & \textbf{$seed_1$} & \textbf{$seed_2$} & \textbf{$seed_3$} & \textbf{$seed_4$} \\
            \midrule
            \textbf{$c_1$} & \checkmark     & \checkmark     & \checkmark     & \checkmark \\
            \textbf{$c_2$} & \checkmark     & \checkmark     & \checkmark     & \checkmark \\
            \textbf{$c_3$} & \checkmark     & \checkmark     & \checkmark     & \checkmark \\
            \textbf{$c_4$} &      & \checkmark     & \checkmark     & \checkmark \\
            \textbf{$c_5$} &      & \checkmark     &     &  \\
            \bottomrule
            \end{tabular} }
	\end{minipage}
\end{table*} 

\textbf{Results and Analysis} The ablation results are presented in Table~\ref{tab1:all_run}. It is elucidated that regardless of the component subjected to ablation, a substantial decline in final performance ensues, underscoring the indispensability of all components for the final efficacy of our method. Crucially, our findings demonstrate that when only the intrinsic rewards part is ablated, namely \textbf{Ours w/o IR}, the performance results exhibit minimal influence, particularly in Mujoco environment tasks. This observation unveils the pivotal role of state representation component in our method, highlighting its significant contribution to the primary performance enhancement and further substantiating the assertions made in Section~\ref{section:intro}. Furthermore, since LESR requires that the input of networks be the concatenation of the source state and the generated state representations (i.e., $s_t^c=(s_t, s_t^r)$), we have conducted experiments to substantiate the indispensability of the source state $s_t$. As depicted in Figure~\ref{fig:drop_src}, it is evident that the source state and the generated state representations work synergistically, 
both playing pivotal roles in achieving optimal performance.

\begin{figure}[h]
    \vspace{10pt}
	\centering
	\includegraphics[width=0.82\linewidth]{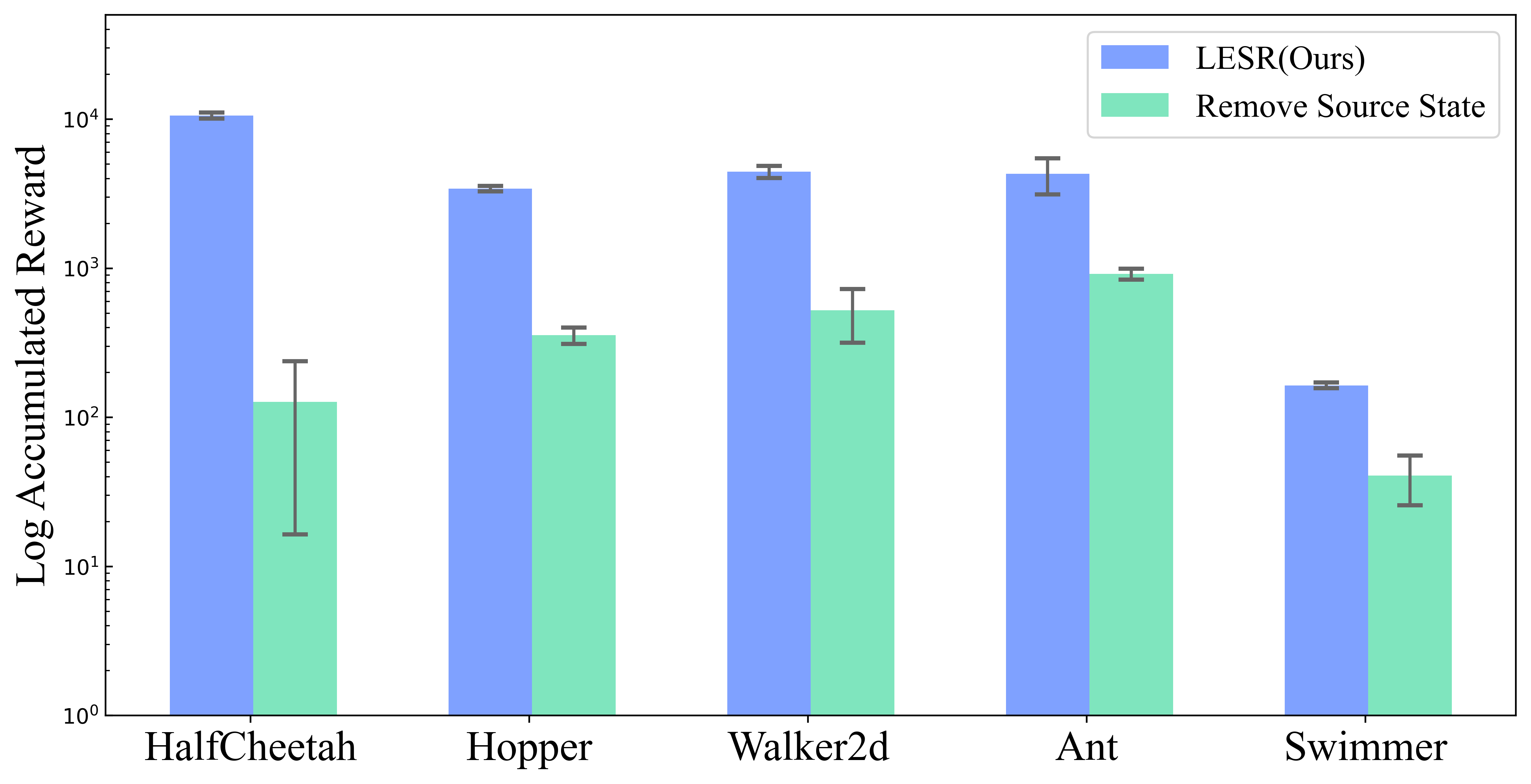}
	\caption{Comparison of Mujoco tasks between LESR and the removal of source state during training (i.e., utilizing only the ($\mathcal{F}(s)$) as input for policy and critic network training). The y-axis is on a logarithmic scale, and error bars represent 5 random seeds.} 
	\vspace{-5pt}
	\label{fig:drop_src}
\end{figure}

Besides, in Appendix~\ref{appen:futher_ablations}, we also validate the role of the Lipschitz constant of LESR through more ablation experiments. Furthermore, we showcase LESR's robustness through experiments solely utilizing intrinsic reward functions, affirming its reliability. Additionally in Appendix~\ref{appen:futher_ablations} experiments on the novel tasks 'Walker Jump' and 'Walker Split Legs' underscore LESR's adaptability to new scenarios. 

\subsection{Directly Transfer to Other Algorithms }\label{sec:algorithm_transfer}

As there is no algorithm-specific information provided in the iteration prompts, we hypothesize that the state representation and intrinsic reward functions are algorithm-agnostic. This suggests that they can be directly transferred and integrated with other RL algorithms without the iteration process in Algorithm~\ref{alg:ours}. 
To substantiate this hypothesis, we retain the best state representation and intrinsic reward function $\mathcal{F}_{best}, \mathcal{G}_{best}$ and combine them with two other widely employed RL algorithms \textbf{PPO}~\cite{schulman2017proximal} and \textbf{SAC}~\cite{haarnoja2018soft} for validation.

The results in Table~\ref{tab:algotransfer} highlight that the state representations and their associated intrinsic reward functions, acquired through the training of \textbf{TD3}, exhibit the potential to be integrated with alternative algorithms, still resulting in improved outcomes. This validates our initial hypothesis. Furthermore, these results further emphasize the efficacy and adaptability of our approach. Consequently, by simply employing a fundamental algorithm to explore state representation and intrinsic reward functions for a given task, it becomes possible to significantly diminish the training complexity associated with that task for other algorithms.

\subsection{Semantic Analysis and Consistency Verification}\label{sec:semantic_analysis}

To elucidate the precise function of the state representations and comprehend the rationale behind their superior performance, we meticulously analyze the state representation functions produced by LLM. We use the \texttt{Swimmer} task from Mujoco as an example for semantic analysis. 
Please refer to Appendix~\ref{appendix:semantic} for details about the state representation functions generated by LLM. 

\begin{figure}[t]
    \vspace{-5pt}
	\centering
	\includegraphics[width=\linewidth]{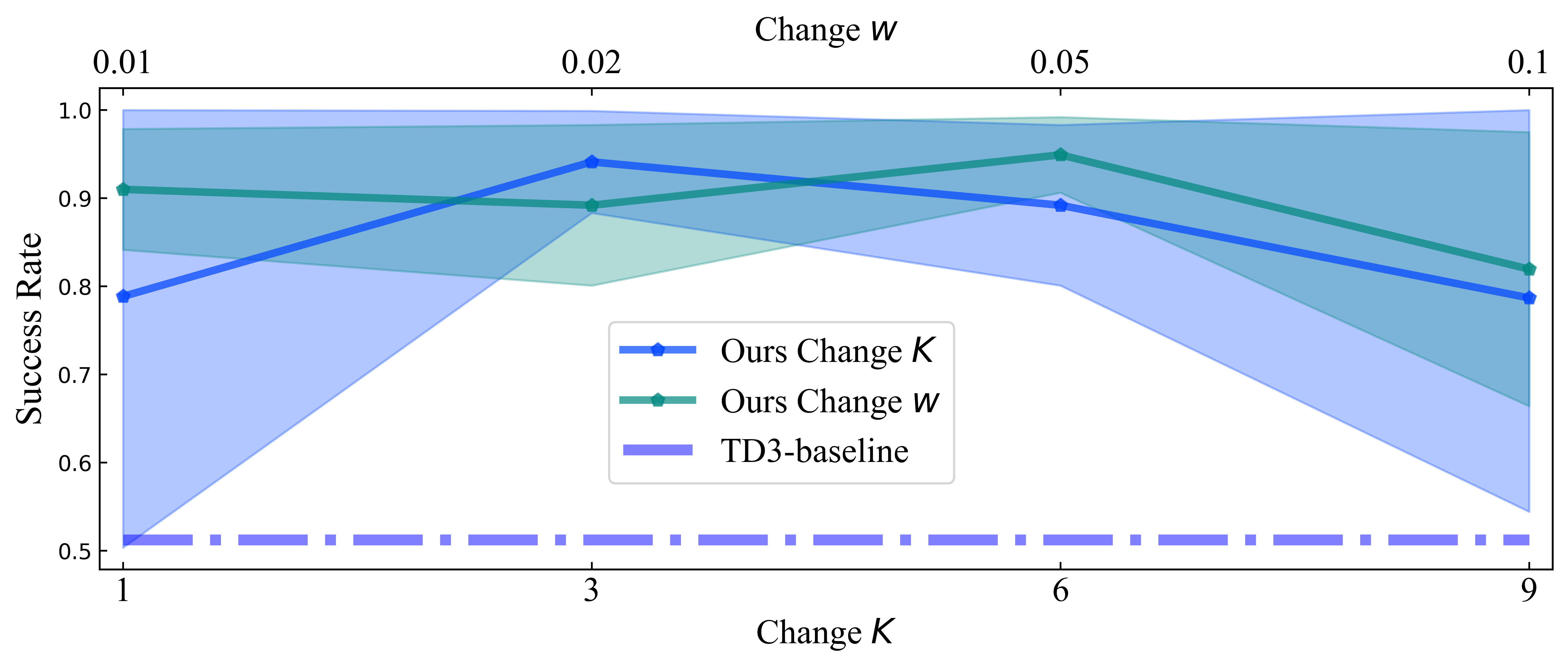}
	\caption{Experimental results for hyperparameter variations in \texttt{AdroitHandDoor}. The sample count $K$ is adjusted to [1, 3, 6, 9]~(bottom x-axis), and the intrinsic reward weight $w$ is modified to [0.01, 0.02, 0.05, 0.1]~(top x-axis). }
	\vspace{-15pt}
	\label{fig:robost-1}
\end{figure}

\textbf{Semantic Analysis} In the case of the \texttt{Swimmer} task, the goal is to move as fast as possible towards the right by applying torque on the rotors and using the fluids friction, without taking any drastic action. The state representations in this task can be categorized into five groups, whose names are derived from the comments provided by LLM: \textit{cosine or sine of the angles~(c1)}, \textit{relative angles between adjacent links~(c2)}, \textit{kinetic energy~(c3)}, \textit{distance moved~(c4)}, and \textit{sum of torques applied~(c5)}. It is apparent that these representations exhibit a strong correlation with the task's objective: \textit{c1} and \textit{c2} contribute to enhanced understanding of posture, \textit{c3} and \textit{c5} signify the necessity to avoid excessive actions, while \textit{c4} corresponds directly to the objective of advancing towards the target.

\textbf{Consistency of LLM} Through semantic analysis, it is demonstrable that LLM is capable of generating significantly task-related state representations of a physically meaningful nature, stemming from the source state. Then, how consistent are the answers from LLM? We carried out experiments with four random seeds for the \texttt{Swimmer} task, and took out the final state representation functions in each experiment for statistics. As demonstrated in Table~\ref{tab:consis}, the functions generated by LLM exhibit a pronounced level of consistency across diverse experiments, e.g. \textbf{$c_1$}, \textbf{$c_2$} and \textbf{$c_3$} are included in all experiments. This reaffirms the importance and universality of the extended state while ensuring the robustness and stability of our methodology.

\subsection{Robustness}\label{sec:robostness}

To validate the robustness and stability of varying hyperparameters, we conducted experiments in the \texttt{AdroitHandDoor} environment. We systematically altered the values of two key hyperparameters: the weight of the intrinsic reward $w$ in Equation~\eqref{eq:max_obj_2} and the sample count $K$. For each experiment, we modified one hyperparameter while keeping the others. The outcomes of these experiments are illustrated in Figure~\ref{fig:robost-1}.

The results indicate that an increase in the sample count $K$ can lead to a performance improvement, for instance, from $K = 1$ to $K = 3$. However, elevating the value of $K$ from $3$ to $9$ does not yield further performance enhancement. Regarding the variation in the weight of the intrinsic reward $w$, the findings illustrate that the final performance is not significantly affected by changes in $w$ and remains stable. In summary, the final performance across all hyperparameter tuning experiments remains consistently within a stable range, significantly surpassing the baseline. This observation underscores the stability of our approach.

\section{Conclusion}

In this paper, we introduce \textbf{LESR}, an algorithm leveraging the coding proficiency and interpretive capacity for physical mechanism of LLM to generate state representation and intrinsic reward function codes for reinforcement learning. Demonstrating its efficacy on benchmarks Mujoco and Gym-Robotics, we illustrate LLM's capability to produce task-specific state representations alongside meaningful intrinsic reward functions. These representations enhance Lipschitz continuity in networks, resulting in superior efficiency and outperforming SOTA baselines. In-depth ablations and additional experiments show the consistency and robustness of LESR. We believe that LESR can effectively contribute to various real-world interaction tasks.

However, our work still suffers limitations. A primary constraint lies in our attempt to derive task-related representations solely from the source state features using LLM, without incorporating external information. This approach may restrict the information available to the network in partially observable environments. Besides, the quality of state representations generated by LLM is constrained by its capabilities, and there is no absolute guarantee that it can produce more task-related representations. This limitation is anticipated to be mitigated as LLM evolves. In the future, we are interested in exploring the integration of additional information to establish a more comprehensive framework for state representations. We anticipate that our work may serve as an inspiration for further exploration in this promising area among researchers. 

\section*{Impact Statement}
This paper presents work whose goal is to advance the field of Machine Learning. There are many potential societal consequences of our work, none which we feel must be specifically highlighted here.

\section*{Acknowledgment}
This work was supported by the National Key R\&D Program of China under Grant 2018AAA0102801.


\bibliography{example_paper}
\bibliographystyle{icml2024}

\newpage
\appendix
\onecolumn

\section{Theoretical Analysis}\label{appendix:proof}

\begin{definition}\label{defi:sigma}
	We firstly define the identity projection and closet point projection map $\mbox{Id}:\mathcal{X} \to \mathcal{X}, \sigma_{\mathcal{X}_k}:\mathcal{X} \to \mathcal{X}_k = \{x_1,\dots,x_n\}$ that satisfies
	\begin{equation}\label{eq:sigma}
		\begin{aligned}
			\forall x\in \mathcal{X}, \mbox{Id}(x) &= x \\
			\forall x\in \mathcal{X}, \|x-\sigma_{\mathcal{X}_k}\|_2 &= \min_{1\leq i\leq n}\{ \|x-x_i\|_2 \}
		\end{aligned}
	\end{equation}
\end{definition}

Now we investigate the convergence under the mapping $f$ in Theorem~\ref{theorem:convergence}:

\begin{theorem}[Convergence]\label{theorem:convergence1}
	If $u_1\in\mathcal{U}_1$ is any minimizer of $\mathcal{L}(u, \mathcal{X}_1)$, where $\mathcal{U}_1$ denotes the solution set of $\mathcal{L}(u, \mathcal{X}_1)$ in Definition~\ref{eq:emprical_loss}, then for any $t>0$, there exists $C \geq 1$ and $0 < c < 1$ satisfies the following with probability at least $1-Ct^{-1}N^{-(ct-1)}$:
	\begin{equation} 
		\mathbb{E}_{x \sim \rho_f} \|u_1^* - u_1\|_2 \leq C\Big[ \Lp(u_1^*) + \Lp(u_1) \Big] \left( \frac{t\log N}{N} \right)^{1/d}
	\end{equation}
\end{theorem}

\begin{proof}[\textit{Proof of Theorem~\ref{theorem:convergence1}}]

Since $u_1\in\mathcal{U}_1$ is a minimizer of $\mathcal{L}(u, \mathcal{X}_1)$, we must have $u_1(x_i)=u_1^*(x_i)$ for all $1 \leq i \leq N$. Then for any $x\in X$ we have
\begin{equation}\label{eq:last_to_distance}
	\begin{aligned}
		\|u_1^*(x) - u_1(x)\|_2
		&=\|u_1^*(x) - u_1^*(\sigma_{\mathcal{X}_1}(x)) + \underbrace{u_1^*(\sigma_{\mathcal{X}_1}(x)) - u_1(\sigma_{\mathcal{X}_1}(x))}_{=0} + u_1(\sigma_{\mathcal{X}_1}(x)) - u_1(x)\|_2\\
		&\leq \|u_1^*(x) - u_1^*(\sigma_{\mathcal{X}_1}(x))\|_2 + \|u_1(\sigma_{\mathcal{X}_1}(x)) - u_1(x)\|_2\\
		&\leq \Big[ \Lp(u_1^*) + \Lp(u_1) \Big] \| x - \sigma_{\mathcal{X}_1}(x)\|_2
	\end{aligned}
\end{equation}

Now we provide a bound between the identity projection and closet point projection:
\begin{lemma}[Lemma 2.9. in \cite{oberman2018lipschitz}]\label{lemma:distance}
	For any $t>0$, the following holds with probability at least $1-Ct^{-1}N^{-(ct-1)}$:
	\begin{equation} 
		\mathbb{E}_{x \sim \rho_f} \| x - \sigma_{\mathcal{X}_1}(x)\|_2 \leq C\left(\frac{t\log N}{N}\right)^{1/d}
	\end{equation}
\end{lemma}

The proof is completed by combining Lemma~\ref{lemma:distance} into Eq~\ref{eq:last_to_distance}.
\end{proof}

Next, we can prove that the generalization loss converges based on Theorem~\ref{theorem:convergence1}:
\begin{theorem}\label{theorem:convergence2}
	Assume that for some $q\geq1$ the loss $\ell$ in Definition~\ref{eq:emprical_loss} satisfies $\ell(y_i,y_k) \leq C \|y_i-y_k\|^q_2 \ \ \text{for all }y_i,y_k\in \mathcal{Y}.$
	Then under Theorem~\ref{theorem:convergence1}, the following bound of the loss $\mathcal{L}[u_1, \mathcal{X}]$ in Definition~\ref{eq:emprical_loss} holds with probability at least $1-Ct^{-1}N^{-(ct-1)}$:
	\begin{equation} 
		\mathcal{L}[u_1, \mathcal{X}] \leq C\Big[ \Lp(u_1^*) + \Lp(u_1) \Big]^q \left( \frac{t\log N}{N} \right)^{q/d}
	\end{equation}
\end{theorem}

\begin{proof}[Proof of Theorem~\ref{theorem:convergence2}]
	We can bound the loss as:
	\begin{equation}\label{eq:last_loss}
		\mathcal{L}[u_1, \mathcal{X}] = \int_{x\in\mathcal{X}} \rho_f(x)\ell\Big(u_1^*(x), u_1(x)\Big) dx\leq C \mathbb{E}_{x \sim \rho_f} \|u_1^* - u_1\|^q_2
	\end{equation}
The proof is completed by combining Theorem~\ref{theorem:convergence1} into Eq~\ref{eq:last_loss}.
\end{proof}

Now we turn to the proof of Theorem~\ref{theorem:convergence}, we start with the following lemma:
\begin{lemma}\label{lemma1}
	Under Assumption~\ref{assumption}, $\rho$ and $\rho_f$ denote the source probability distribution on $\mathcal{X}$ and probability distribution on $f(\mathcal{X})$:
	\begin{equation}
		\begin{aligned}
			\mathbb{E}_{x \sim \rho_f} \| x - \sigma_{\mathcal{X}_1}(x)\|_2 = \mathbb{E}_{x \sim \rho} \| x - \sigma_{\mathcal{X}_0}(x)\|_2
		\end{aligned}
	\end{equation}
\end{lemma}

\begin{proof}[\textit{Proof of Lemma~\ref{lemma1}}]

Let $\mathcal{X}_2 = \{x | x \in \mathcal{X}, x \notin \mathcal{X}_0\}$, from Definition~\ref{defi:sigma}, when $x \in \mathcal{X}_0$, $\sigma_{\mathcal{X}_0}(x) = x$, then we have:
\begin{equation}
	\begin{aligned}
		\int_{x \in \mathcal{X}} \rho(x)\| x - \sigma_{\mathcal{X}_0}(x)\|_2 \ dx 
		&= \int_{x \in \mathcal{X}_2} \rho(x)\| x - \sigma_{\mathcal{X}_0}(x)\|_2 \ dx + \underbrace{\int_{x \in \mathcal{X}_0} \rho(x)\| x - \sigma_{\mathcal{X}_0}(x)\|_2 \ dx}_{ = 0}\\
		&= \int_{x \in \mathcal{X}_2} \rho(x)\| x - \sigma_{\mathcal{X}_0}(x)\|_2 \ dx
	\end{aligned}
\end{equation}

Consider the mapping $f:\mathcal{X} \rightarrow \mathcal{X}$ only swaps the order of $x$ in the dataset $\mathcal{X}_0$, which means:
\begin{equation}
	\begin{aligned}
		f(x) &= \begin{cases}
			x' \Big| x, x' \in \mathcal{X}_0\\
			x \ \Big| x\in \mathcal{X}_2
		\end{cases}\\
		\mathcal{X}_1 = \{f(x_i) | f(x_i) &\in \mathcal{X}_0, i=1,\dots,N\}\mbox{ and }
		 \mathcal{X}_0 = \mathcal{X}_1
	\end{aligned}
\end{equation}

Therefore we get $\forall x\in \mathcal{X}_2, \rho_f(x) = \rho(x)$. Hence:
\begin{equation}
	\begin{aligned}
		\mathbb{E}_{x \sim \rho} \| x - \sigma_{\mathcal{X}_0}(x)\|_2 
		&= \int_{x \in \mathcal{X}_2} \rho(x)\| x - \sigma_{\mathcal{X}_0}(x)\|_2 \ dx \\ 
		&= \int_{x \in \mathcal{X}_2} \rho_f(x)\| x - \sigma_{\mathcal{X}_0}(x)\|_2 \ dx
		= \mathbb{E}_{x \sim \rho_f} \| x - \sigma_{\mathcal{X}_1}(x)\|_2
	\end{aligned}
\end{equation}

\end{proof}

\begin{proof}[Proof of Theorem~\ref{theorem:convergence}]

Since $u_0\in\mathcal{U}_0$ is a minimizer of $\mathcal{L}(u, \mathcal{X}_0)$, we must have $u_0(x_i)=u_0^*(x_i)$ for all $1 \leq i \leq n$. Then for any $x\in X$ we have
\begin{equation}
	\begin{aligned}
	\|u_0^*(x) - u_0(x)\|_2
	&=\|u_0^*(x) - u_0^*(\sigma_{\mathcal{X}_0}(x)) + u_0^*(\sigma_{\mathcal{X}_0}(x)) - u_0(\sigma_{\mathcal{X}_0}(x)) + u_0(\sigma_{\mathcal{X}_0}(x)) - u_0(x)\|_2\\
	&\leq \|u_0^*(x) - u_0^*(\sigma_{\mathcal{X}_0}(x))\|_2 + \|u_0(\sigma_{\mathcal{X}_0}(x)) - u_0(x)\|_2\\
	&\leq \Big[ \Lp(u_0^*) + \Lp(u_0) \Big] \| x - \sigma_{\mathcal{X}_0}(x)\|_2
	\end{aligned}
\end{equation}

Therefore, combined with Theorem~\ref{theorem:convergence1}, we have:
\begin{equation}
	\begin{aligned}
		\sup_{u_1\in\mathcal{U}_1} \mathbb{E}_{x \sim \rho_f}\|u_1^* - u_1\|_2
		&= \Big[ \Lp(u_1^*) + \Lp(u_1) \Big] \mathbb{E}_{x \sim \rho}\| x - \sigma_{\mathcal{X}_1}(x)\|_2 \\
		&\leq \Big[ \Lp(u_0^*) + \Lp(u_0) \Big] \mathbb{E}_{x \sim \rho_f}\| x - \sigma_{\mathcal{X}_1}(x)\|_2 \\
		&= \sup_{u_0\in\mathcal{U}_0} \mathbb{E}_{x \sim \rho}\|u_0^* - u_0\|_2\ \ (\mbox{invoking Assumption~\ref{assumption}, Lemma~\ref{lemma1}})
	\end{aligned}
\end{equation} \end{proof}

\section{Why Lipschitz constant is crucial for RL}\label{appen:why_crucial}

In this section we main focus on the relationship between the Lipschitz constant of the reward function and the continuity of the value function in RL. Firstly we make some assumptions. 

\begin{definition} \label{defi:value}
	In reinforcement learning, given a policy $\pi$, $\gamma$ denotes the discounted factor, $r$~($r: s\to \mathbb{R}$) denotes the reward function, $H$ denotes the length of trajectory, the definition of the value function $V(s)$ is:
	\begin{equation}
		\begin{aligned}
			 V^\pi(s) = \mathbb{E}\big[ \sum_{t = 0}^{H} \gamma^tr |s_0 = s, a_t \sim \pi(s_t) \big].
	 	\end{aligned}
 	\end{equation}
\end{definition}

Similar to~\citet{farahmand2017value}, we make the deterministic assumption of the environment and RL policy. Besides, the policy of RL algorithm TD3 utilized in our method is also deterministic.
\begin{assumption}\label{assu:deterministic}
	The environment transition and policy are deterministic.
\end{assumption} 

We also make the assumptions of the Lipschitz constant of reward function and environment dynamic transition which is similar to previous work~\citep{asadi2018lipschitz}.
\begin{assumption}\label{assu:reward_lip}
	There exists constants $K_1, K_2$ such that the lipschitz constant of reward function $r$~($r: s\to \mathbb{R}$) is $K_1$, where $s\in \mathcal{S}$ is the state. In other words, $\Lp(r) = K_1$, and $\mathcal{P}(s, a):\mathcal{S}\times \mathcal{A}\to \mathcal{S}$ denotes the environment dynamic transition function, $K_2$ that satisfies: $\Lp(\mathcal{P}) = K_2$.
\end{assumption}  

$\bullet$ \textbf{Relationship between $\Lp(r; \mathcal{S})$ and $\Lp(V; \mathcal{S})$} 

Drawing upon the aforementioned definitions and assumptions, let us commence our analysis:
\begin{theorem}\label{theo:value_lip}
    Under Assumption~\ref{assu:deterministic}, \ref{assu:reward_lip}, given a RL policy $\pi$, the value function of $\pi$, namely $V^\pi$, satisfies: 
    \begin{equation}
    	\begin{aligned}
    		\forall s_1, s_2 \in \mathcal{S},\ \ \|V^\pi(s_1) - V^\pi(s_2)\| \leq \frac{[1 - (\gamma K_2)^H]K_1}{1 - \gamma K_2} \| s_1 - s_2 \|.
    	\end{aligned}
    \end{equation} 
\end{theorem} 

\begin{proof} [Proof of Theorem~\ref{theo:value_lip}]
	Because the policy $\pi$ and environment dynamic transition function $\mathcal{P}$ are deterministic, for $\forall s_1, s_2 \in \mathcal{S}$, two trajectories $\tau_1, \tau_2$ can be generated:
	\begin{equation}
		\begin{aligned}
			\tau_1 &= \{ s_{1,i} | t = 0,\dots,H \ \ s_{1,0} = s_1,\ \  s_{1, t+1} = \mathcal{P}(s_{1, t}, \pi(s_{1, t})) \}, \\
			\tau_2 &= \{ s_{2,i} | t = 0,\dots,H \ \ s_{2,0} = s_2,\ \  s_{2, t+1} = \mathcal{P}(s_{2, t}, \pi(s_{2, t})) \}. \\
		\end{aligned}
	\end{equation}
	
	Under Definition~\ref{defi:value}:
	\begin{equation}  \label{eq:b2}
		\begin{aligned}
			\bigg\| V^\pi(s_1) - V^\pi(s_2) \bigg\| 
			&= \bigg\| \sum_{t = 0}^{H} \gamma^tr(s_{1, t}) - \sum_{t = 0}^{H} \gamma^tr(s_{2, t}) \bigg\| \\
			&= \bigg\| \sum_{t = 0}^{H} \gamma^t\big[ r(s_{1, t}) - r(s_{2, t}) \big] \bigg\| \\
			&\leq \sum_{t = 0}^{H} \gamma^t \bigg\|  \big[ r(s_{1, t}) - r(s_{2, t}) \big] \bigg\|  \ \ \mbox{\# invoking Assumption~\ref{assu:reward_lip}} \\
			&\leq \sum_{t = 0}^{H} \gamma^t K_1 \big\|  s_{1, t} - s_{2, t} \big\|.
		\end{aligned}
	\end{equation}
	
	The state dynamic distances $\big\|  s_{1, t} - s_{2, t} \big\|$ can be estimated as follows:
	\begin{equation} \label{eq:b1}
		\begin{aligned} 
			\big\|  s_{1, t} - s_{2, t} \big\| 
			\leq K_2 \big\|  s_{1, t - 1} - s_{2, t - 1} \big\| 
			\leq K_2^2 \big\|  s_{1, t - 2} - s_{2, t - 2} \big\| 
			\dots 
			\leq K_2^t \big\|  s_1 - s_2 \big\|.
		\end{aligned}
	\end{equation}
	
	Combined Equation~\ref{eq:b1} with \ref{eq:b2}:
	\begin{equation} 
		\begin{aligned}
			\bigg\| V^\pi(s_1) - V^\pi(s_2) \bigg\| 
			\leq \sum_{t = 0}^{H} \big(\gamma K_2\big)^t K_1 \big\|  s_1 - s_2 \big\| 
			= \frac{[1 - (\gamma K_2)^H]K_1}{1 - \gamma K_2} \| s_1 - s_2 \|.
		\end{aligned}
	\end{equation}
\end{proof}

Proof of Theorem~\ref{theo:value_lip} is completed. A direct corollary from Theorem~\ref{theo:value_lip} is:
\begin{corollary}\label{coro:value_lip} 
	\begin{equation}
		\begin{aligned}
			\Lp(V^\pi; \mathcal{S}) \leq  \frac{[1 - (\gamma K_2)^H]K_1}{1 - \gamma K_2}.
		\end{aligned}
	\end{equation}
\end{corollary}

Corollary \ref{coro:value_lip} demonstrates that reducing the Lipschitz constant $K_1$ of the reward function directly decreases the upper bound of the Lipschitz constant of value functions in RL. This confirms that our approach effectively improves the continuity of the value functions. 

$\bullet$ \textbf{How does $\Lp(V; \mathcal{S})$ work for RL algorithms?}


We provide additional definitions to aid in understanding the subsequent theorems. Inspired by~\citet{farahmand2017value, zheng2023model, farahmand2010error, asadi2018lipschitz}: 

\begin{definition} \label{defi:V}
	Denote environment dynamic transition function $\mathcal{P}$ defined in ~\ref{theo:value_lip}. For a given value function $V$, the Bellman operator as $\mathcal{T}$, policy transition matrix as $P^\pi: \mathcal{S} \times \mathcal{S} \to \mathbb{R}$, the greedy policy $\pi_g$ of $V$, and the optimal policy $\pi^*$ and optimal value funtion $V^*$ are defined as follows:
	\begin{equation}
		\begin{aligned}
			&\mathcal{T}V(s) 
			= r(s) + \max_a \mathcal{P}(s, a) V(\mathcal{P}(s, a)), \\
			&\mathcal{T}^\pi V = r + \gamma P^\pi V, \\
			&P^\pi(s_1, s_2) = \begin{cases}
				&0, \ \ \ s_2 \neq \mathcal{P}(s_1, \pi(s_1)) \\
				&1, \ \ \ s_2 = \mathcal{P}(s_1, \pi(s_1)) \\
			\end{cases}, \\
			&\forall s \in \mathcal{S}, \ \ \pi_g = \arg\max_a \bigg[ r(s) + \mathcal{P}(s, a) V(\mathcal{P}(s, a)) \bigg], \\
			&\forall s \in \mathcal{S}, \ \ V^{\pi^*}(s) = V^*(s) = \max_\pi V^\pi(s). \\
		\end{aligned}
	\end{equation}
\end{definition}

\begin{definition} \label{defi:c}
	Denote $\pi_1, \pi_2, \dots, \pi_m$ as any m policies arranged in any order, $\forall m \ge 1$, $u$ and $v$ are two state distributions, define $C(u) \in \mathbb{R}^+ \cup {+\infty}$, $c(m) \in \mathbb{R}^+ \cup {+\infty}$ and $c(0) = 1$:
	\begin{equation}
		\begin{aligned} 
			c(m) &= \max_{\pi_1, \pi_2, \dots\pi_m, s\in \mathcal{S}} \frac{\Big( vP^{\pi_1}P^{\pi_2}\dots P^{\pi_m} \Big)(s)}{u(s)}, \\
			C_1(v, u) &= (1 - \gamma) \sum_{m\ge 0} \gamma^m c(m). \\ 
		\end{aligned}
	\end{equation}
\end{definition} 

$c(m)$, as defined in Definition~\ref{defi:c}, quantifies the discrepancy between transitions of $m$ policies from an initial state distribution $v$ over $m$ time steps and a target distribution $u$, while $C_1(v, u)$ quantifies the discounted discrepancy. 

\begin{theorem}\label{theo:value_lip_convergence}
    Denote $\pi_g$ as the greedy policy of a given value function $V$, $V$ is the value function of some policy $\pi_0$.  $u$ and $v$ are two state distributions. The norm $\|V\|_{q, u}$ is the $q$-norm of $V$ weighted by $u$ which is defined as $\|V\|_{q, u} = \big[\sum_s u(s)\big(V(s)\big)^q\big]^{\frac{1}{q}}$. Suppose $\forall s_1, s_2 \in \mathcal{S}$, there exists a constant $D_1$ that satisfies $\| s_1 - s_2 \| \le D_1$ and under assumption~\ref{assu:deterministic}, \ref{assu:reward_lip}: 
    \begin{equation}
    	\begin{aligned} 
    		& ||V^* - V^{\pi_g}||_{q, v} \le \frac{2\gamma K_1 D_1 \big(1 - (\gamma K_2)^H\big)}{(1 - \gamma)(1 - \gamma K_2)} \big[C_1(v, u)\big]^{\frac{1}{q}}. \\
    	\end{aligned}  
    \end{equation} 
\end{theorem} 

\begin{proof}[Proof of Theorem~\ref{theo:value_lip_convergence}]
	
	We first present two lemmas that aid in the subsequent proof.
	
	\begin{lemma} \label{lema:value_lip_lema1}
		Denote matrix $I$ as the identity matrix(here $V_1\preceq V_2$ means $\forall s \in \mathcal{S}, V_1(s) \le V_2(s)$): 
		\begin{equation}
			\begin{aligned}
				V^* - V^{\pi_g} \preceq \Big[ (I-\gamma P^{\pi^*})^{-1} + (I-\gamma P^{\pi_g})^{-1} \Big] |\mathcal{T}V - V|. \\
			\end{aligned}
		\end{equation}
	\end{lemma}
	\begin{proof}[Proof of Lemma~\ref{lema:value_lip_lema1}]
		It can be derived that $\mathcal{T}^{\pi_g} V = \mathcal{T}V$, $\mathcal{T}^{\pi_g} V^{\pi_g} = V^{\pi_g}$:
		\begin{equation} \label{eq28}
			\begin{aligned}
				V^* - V^{\pi_g} &= \mathcal{T}^{\pi^*}V^* - \mathcal{T}^{\pi^*}V + \underbrace{\mathcal{T}^{\pi^*}V - \mathcal{T}^{\pi_g}V}_{\preceq \bm{0}} + \mathcal{T}^{\pi_g}V - \mathcal{T}^{\pi_g}V^{\pi_g} \\
				&\preceq \mathcal{T}^{\pi^*}V^* - \mathcal{T}^{\pi^*}V + \mathcal{T}^{\pi_g}V - \mathcal{T}^{\pi_g}V^{\pi_g} \\
				&= \gamma P^{\pi^*}(V^* - V^{\pi_g} + V^{\pi_g} - V) + \gamma P^{\pi_g}(V - V^{\pi_g}). \\
				&\Rightarrow (I - \gamma P^{\pi^*})(V^* - V^{\pi_g}) \preceq \gamma (P^{\pi^*} - P^{\pi_g})(V^{\pi_g} - V). \\
				&\Rightarrow V^* - V^{\pi_g} \preceq (I - \gamma P^{\pi^*})^{-1} \gamma (P^{\pi^*} - P^{\pi_g})(V^{\pi_g} - V). \\
			\end{aligned} 
		\end{equation} 
		
		\begin{equation} \label{eq29}
			\begin{aligned}
				(I - \gamma P^{\pi_g})(V^{\pi_g} - V) 
				&= V^{\pi_g} - V - \gamma P^{\pi_g}V^{\pi_g} + \gamma P^{\pi_g}V  \\
				&= r + \gamma P^{\pi_g}V - \Big( r + \gamma P^{\pi_g}V^{\pi_g} \Big) + V^{\pi_g} - V \\
				&= \mathcal{T}^{\pi_g}V - \mathcal{T}^{\pi_g}V^{\pi_g} + V^{\pi_g} - V \\
				&= \mathcal{T}V - V. \\
			\end{aligned} 
		\end{equation}
		
		Combine Inequation~\ref{eq28} and Equation~\ref{eq29}:
		\begin{equation}
			\begin{aligned}
				V^* - V^{\pi_g} 
				&\preceq (I - \gamma P^{\pi^*})^{-1} \gamma (P^{\pi^*} - P^{\pi_g})(V^{\pi_g} - V) \\
				&= (I - \gamma P^{\pi^*})^{-1} \gamma (P^{\pi^*} - P^{\pi_g})(I - \gamma P^{\pi_g})^{-1}(\mathcal{T}V - V) \\
				&= (I - \gamma P^{\pi^*})^{-1}\Big[ (I - \gamma P^{\pi_g}) - (I - \gamma P^{\pi^*}) \Big](I - \gamma P^{\pi_g})^{-1}(\mathcal{T}V - V) \\
				&= \Big[ (I - \gamma P^{\pi^*})^{-1} - (I - \gamma P^{\pi_g})^{-1} \Big](\mathcal{T}V - V) \\
				&\preceq \Big[ (I - \gamma P^{\pi^*})^{-1} + (I - \gamma P^{\pi_g
				})^{-1} \Big]|\mathcal{T}V - V|. \\
			\end{aligned}
		\end{equation}
		
		Proof of Lemma~\ref{lema:value_lip_lema1} is completed.
	\end{proof}
	
	\begin{lemma}  \label{lema:value_lip_lema2}
		For a given policy $\pi$:
		\begin{equation}
			\begin{aligned}
				(I - \gamma P^{\pi})^{-1} = \sum_{k=0}^{\infty} \Big( \gamma P^{\pi} \Big)^k.
			\end{aligned}
		\end{equation}
	\end{lemma} 
	
	\begin{proof}[Proof of Lemma~\ref{lema:value_lip_lema2}]
		\begin{equation}
			\begin{aligned} \label{eq23}
				(I - \gamma P^{\pi})\sum_{k=0}^{\infty} \Big( \gamma P^{\pi} \Big)^k 
				&= \sum_{k=0}^{\infty} \bigg[ \Big( \gamma P^{\pi} \Big)^k - \Big( \gamma P^{\pi} \Big)^{k+1} \bigg] \\
				&= I - \lim_{k\rightarrow\infty} \Big( \gamma P^{\pi} \Big)^{k+1} = I. \\
			\end{aligned} 
		\end{equation}
		Proof of Lemma~\ref{lema:value_lip_lema2} is completed.
	\end{proof} 
	
	Now we turn to the proof of Theorem~\ref{theo:value_lip_convergence}, we rewrite Lemma~\ref{lema:value_lip_lema1} as follows:
	\begin{equation} \label{eq:defi_A}
		\begin{aligned} 
			V^* - V^{\pi_g} &\preceq \frac{2}{1 - \gamma} A |\mathcal{T}V - V| \\
			A &= \frac{1 - \gamma}{2} \Big[ (I-\gamma P^{\pi^*})^{-1} + (I-\gamma P^{\pi_g})^{-1} \Big]. \\
		\end{aligned}
	\end{equation}
	
	Under the $q$-norm weighted by $v$:
	\begin{equation} \label{eq:jensen}
		\begin{aligned} 
			\|V^* - V^{\pi_g}\|_{q, v}^q &\le [\frac{2}{1 - \gamma}]^q \sum_{s\in \mathcal{S}} v(s) \big[ A(s) |\mathcal{T}V(s) - V(s)| \big]^q \\ 
			&\le [\frac{2}{1 - \gamma}]^q \sum_{s\in \mathcal{S}} v(s) A(s) |\mathcal{T}V(s) - V(s)|^q.
		\end{aligned}
	\end{equation}
	
	Equation~\ref{eq:jensen} stems from Jensen's inequality, exploiting the convexity of the function $f(x) = x^q$ for $q \geq 1$. Furthermore, when combined with Lemma~\ref{lema:value_lip_lema2} and Equation~\ref{eq:defi_A}, each row of matrix $A$ attains a sum of $\frac{1 - \gamma}{2}\left[\frac{1}{1 - \gamma} + \frac{1}{1 - \gamma}\right] = 1$, meeting the prerequisites of Jensen's inequality.
	
	Next we focus on the $v(s)A(s)$ in Equation~\ref{eq:jensen}:
	\begin{equation}\label{eq:vA}
		\begin{aligned}
			vA
			&= \frac{1 - \gamma}{2} v \Big[ (I-\gamma P^{\pi^*})^{-1} + (I-\gamma P^{\pi_g})^{-1} \Big] \\
			&= \frac{1 - \gamma}{2} \sum_{m=0}^{\infty} \gamma^m \Big(  vP^{\pi^*} + vP^{\pi_g} \Big) \ \ \mbox{\# invoking Lemma~\ref{lema:value_lip_lema2}} \\
			&\preceq \frac{1 - \gamma}{2} \Big( \sum_{m\ge 0} \gamma^m c(m)u + \sum_{m\ge 0} \gamma^m c(m)u \Big) \ \ \mbox{\# under Definition~\ref{defi:c}} \\
			&= C_1(v, u)u. \\
		\end{aligned}
	\end{equation}
	
	Now we return to Equation~\ref{eq:jensen}, combined with Equation~\ref{eq:vA}:
	\begin{equation} \label{eq:42}
		\begin{aligned}
			\|V^* - V^{\pi_g}\|_{q, v}^q 
			&\le [\frac{2}{1 - \gamma}]^q C_1(v, u) \sum_{s\in \mathcal{S}} u(s) |\mathcal{T}V(s) - V(s)|^q. \\ 
		\end{aligned}
	\end{equation}
	
	In the description of Theorem~\ref{theo:value_lip_convergence}, Because $\pi_g$ is denoted as the greedy policy of a given value function $V$, $V$ is the value function of some policy $\pi_0$. Therefore under assumption~\ref{assu:deterministic}: $\mathcal{T}V(s) - V(s) = r(s) + \gamma V(s_1) - r(s) - \gamma V(s_2) = \gamma (V(s_1) - V(s_2))$, $s_1 = \mathcal{P}(s, \pi_g(s)), s_2 = \mathcal{P}(s, \pi_0(s))$. 
	
	Continuing Equation~\ref{eq:42}:
	\begin{equation} 
		\begin{aligned}
			\|V^* - V^{\pi_g}\|_{q, v}^q 
			&\le [\frac{2}{1 - \gamma}]^q C_1(v, u) \sum_{s\in \mathcal{S}} u(s) |\mathcal{T}V(s) - V(s)|^q \\
			&=  [\frac{2}{1 - \gamma}]^q C_1(v, u) \sum_{s\in \mathcal{S}} u(s) |\gamma (V(s_1) - V(s_2)|^q \ \ \mbox{\# ($s_1 = \mathcal{P}(s, \pi_g(s)), s_2 = \mathcal{P}(s, \pi_0(s))$)} \\
			&\le [\frac{2}{1 - \gamma}]^q C_1(v, u) \sum_{s\in \mathcal{S}} u(s) \bigg[\gamma\Lp(V; \mathcal{S})\|s_1 - s_2\| \bigg]^q \ \ \mbox{\# invoking Corollary~\ref{coro:value_lip}} \\
			&\le \bigg[ \frac{2 \gamma}{1 - \gamma} \frac{\big(1 - (\gamma K_2)^H\big)K_1 D_1}{1 - \gamma K_2} \bigg]^q C_1(v, u). \\
		\end{aligned}
	\end{equation}
	Proof of Theorem~\ref{theo:value_lip_convergence} is completed.
\end{proof}

A direct corollary from Theorem~\ref{theo:value_lip_convergence} is:
\begin{corollary}\label{coro:value_lip_convergence} 
	Under the setting of Theorem~\ref{theo:value_lip_convergence}:
	\begin{equation}
		\begin{aligned} 
			& ||V^* - V^{\pi_g}||_\infty \le \frac{2\gamma K_1 D_1 \big(1 - (\gamma K_2)^H\big)}{(1 - \gamma)(1 - \gamma K_2)}. 
		\end{aligned}
	\end{equation}
\end{corollary}

\section{All of Our Prompts}

\label{appendix:prompts}

There are three prompt templates in total. 

\textbf{The initial prompt for the first iteration.} \quad This prompt is designed to use at the commencement of the first iteration, which is the letter `$p$' in Algorithm~\ref{alg:ours}. LLM is required to output the state representation and intrinsic reward functions in python code format. Notably, the `task\_description' and `detail\_content of each dimensions' in the prompt are derived from the official document of Mujoco\footnote{https://www.gymlibrary.dev/environments/mujoco/} and Gym-Robotics\footnote{https://robotics.farama.org/}.

\textbf{Prompt for Chain-of-thought Feedback Analysis.} \quad This prompt is formulated to facilitate the examination by LLM of the outcomes from all training experiments during each iteration in a chain-of-thought process\cite{wei2022chain}. This prompt corresponds to the variable $p_{feedback}$ in Algorithm~\ref{alg:ours}. LLM is expected to provide suggestions about how to enhance the performance of the state representation function codes. It is noteworthy that the `iteration\_results' referred to in the prompt encompasses both policy performance and correlation coefficients, as elaborated in Section~\ref{sec:state_feature_expansion}.

\textbf{Subsequent Prompt for Later Iterations.} \quad This prompt is similar to `The initial prompt for the first iteration'. However, the difference is that it contains the information of history iterations, as well as LLM's suggestions about how to enhance the performance of the state representation and intrinsic reward functions. 

Here are the prompt templates:

\begin{mybox}[frametitle=Initial Prompt for the First Iteration]
	Revise the state representation for a reinforcement learning agent. \\
	=========================================================\\
	The agent’s task description is:\\
	\textcolor{red}{\{task\_description\}}\\
	=========================================================\\
	\\
	The current state is represented by a \textcolor{red}{\{total\_dim\}}-dimensional Python NumPy array, denoted as `s`.\\
	\\
	Details of each dimension in the state `s` are as follows:\\
	\textcolor{red}{\{detail\_content\}}\\
	You should design a task-related state representation based on the source \textcolor{red}{\{total\_dim\}} dim to better for reinforcement training, using the detailed information mentioned above to do some caculations, and feel free to do complex caculations, and then concat them to the source state. \\
	\\
	Besides, we want you to design an intrinsic reward function based on the revise\_state python function.\\
	\\
	That is to say, we will:\\
	1. use your revise\_state python function to get an updated state: updated\_s = revise\_state(s)\\
	2. use your intrinsic reward function to get an intrinsic reward for the task: r = intrinsic\_reward(updated\_s)\\
	3. to better design the intrinsic\_reward, we recommond you use some source dim in the updated\_s, which is between updated\_s[0] and updated\_s[\textcolor{red}{\{total\_dim - 1\}}] \\
	4. however, you must use the extra dim in your given revise\_state python function, which is between updated\_s[\textcolor{red}{\{total\_dim\}}] and the end of updated\_s\\

	Your task is to create two Python functions, named `revise\_state`, which takes the current state `s` as input and returns an updated state representation, and named `intrinsic\_reward`, which takes the updated state `updated\_s` as input and returns an intrinsic reward. The functions should be executable and ready for integration into a reinforcement learning environment.\\
	\\
	The goal is to better for reinforcement training. Lets think step by step. Below is an illustrative example of the expected output:\\
	\\
	```python\\
	import numpy as np\\
	def revise\_state(s):\\
	\# Your state revision implementation goes here\\
	return updated\_s\\
	def intrinsic\_reward(updated\_s):\\
	\# Your intrinsic reward code implementation goes here\\
	return intrinsic\_reward\\
	```
\end{mybox}

\begin{mybox}[frametitle=Prompt for Chain-of-thought Feedback Analysis]
	We have successfully trained Reinforcement Learning policy using \textcolor{red}{\{args.sample\_count\}} different state revision codes and intrinsic reward function codes sampled by you, and each pair of code is associated with the training of a policy relatively.\\
	\\
	Throughout every state revision code's training process, we monitored:\\
	1. The final policy performance(accumulated reward).\\
	2. Most importantly, every state revise dim's Lipschitz constant with the reward. That is to say, you can see which state revise dim is more related to the reward and which dim can contribute to enhancing the continuity of the reward function mapping.\\
	\\
	Here are the results:\\
	\textcolor{red}{\{iteration\_results(performance and Lipschitz constants)\}}\\
	\\
	You should analyze the results mentioned above and give suggestions about how to imporve the performace of the "state revision code".\\
	\\
	Here are some tips for how to analyze the results:\\
	(a) if you find a state revision code's performance is very low, then you should analyze to figure out why it fail\\
	(b) if you find some dims' Lipschitz constant very large, you should analyze to figure out what makes it fail\\
	(c) you should also analyze how to imporve the performace of the "state revision code" and "intrinsic reward code" later\\
	\\
	Lets think step by step. Your solution should aim to improve the overall performance of the RL policy.
\end{mybox}

\begin{mybox}[frametitle=Subsequent Prompt for Later Iterations]
	Revise the state representation for a reinforcement learning agent. \\
	=========================================================\\
	The agent’s task description is:\\
	\textcolor{red}{\{task\_description\}}\\
	=========================================================\\
	\\
	The current state is represented by a \textcolor{red}{\{total\_dim\}}-dimensional Python NumPy array, denoted as `s`.\\
	\\
	Details of each dimension in the state `s` are as follows:\\
	\textcolor{red}{\{detail\_content\}}\\
	You should design a task-related state representation based on the source \textcolor{red}{\{total\_dim\}} dim to better for reinforcement training, using the detailed information mentioned above to do some caculations, and feel free to do complex caculations, and then concat them to the source state. \\
	\\
	For this problem, we have some history experience for you, here are some state revision codes we have tried in the former iterations:\\
	\textcolor{red}{\{former\_histoy\}}\\
	\\
	Based on the former suggestions. We are seeking an improved state revision code and an improved intrinsic reward code that can enhance the model's performance on the task. The state revised code should incorporate calculations, and the results should be concatenated to the original state.\\
	\\
	Besides, We are seeking an improved intrinsic reward code.\\
	\\
	That is to say, we will:\\
	1. use your revise\_state python function to get an updated state: updated\_s = revise\_state(s)\\
	2. use your intrinsic reward function to get an intrinsic reward for the task: r = intrinsic\_reward(updated\_s)\\
	3. to better design the intrinsic\_reward, we recommond you use some source dim in the updated\_s, which is between updated\_s[0] and updated\_s[\textcolor{red}{\{total\_dim - 1\}}] \\
	4. however, you must use the extra dim in your given revise\_state python function, which is between updated\_s[\textcolor{red}{\{total\_dim\}}] and the end of updated\_s\\

	Your task is to create two Python functions, named `revise\_state`, which takes the current state `s` as input and returns an updated state representation, and named `intrinsic\_reward`, which takes the updated state `updated\_s` as input and returns an intrinsic reward. The functions should be executable and ready for integration into a reinforcement learning environment.\\
	\\
	The goal is to better for reinforcement training. Lets think step by step. Below is an illustrative example of the expected output:\\
	\\
	```python\\
	import numpy as np\\
	def revise\_state(s):\\
	\# Your state revision implementation goes here\\
	return updated\_s\\
	def intrinsic\_reward(updated\_s):\\
	\# Your intrinsic reward code implementation goes here\\
	return intrinsic\_reward\\
	```	
\end{mybox}

\section{Additional Environment Information}\label{appendix:env_info}

\subsection{Mujoco Environment Tasks}

\textbf{HalfCheetah} \quad This environment is based on the work by P. Wawrzyński in \cite{wawrzynski2009cat}. The HalfCheetah is a 2-dimensional robot consisting of 9 body parts and 8 joints connecting them (including two paws). The goal is to apply a torque on the joints to make the cheetah run forward (right) as fast as possible, with a positive reward allocated based on the distance moved forward and a negative reward allocated for moving backward. The torso and head of the cheetah are fixed, and the torque can only be applied on the other 6 joints over the front and back thighs (connecting to the torso), shins (connecting to the thighs) and feet (connecting to the shins).

\textbf{Hopper} \quad This environment is based on the work done by Erez, Tassa, and Todorov in \cite{erez2011infinite}. The environment aims to increase the number of independent state and control variables as compared to the classic control environments. The hopper is a two-dimensional one-legged figure that consist of four main body parts - the torso at the top, the thigh in the middle, the leg in the bottom, and a single foot on which the entire body rests. The goal is to make hops that move in the forward (right) direction by applying torques on the three hinges connecting the four body parts.

\textbf{Walker} \quad This environment builds on the \textbf{Hopper} environment by adding another set of legs making it possible for the robot to walk forward instead of hop. Like other Mujoco environments, this environment aims to increase the number of independent state and control variables as compared to the classic control environments. The walker is a two-dimensional two-legged figure that consist of seven main body parts - a single torso at the top (with the two legs splitting after the torso), two thighs in the middle below the torso, two legs in the bottom below the thighs, and two feet attached to the legs on which the entire body rests. The goal is to walk in the in the forward (right) direction by applying torques on the six hinges connecting the seven body parts.

\textbf{Ant} \quad This environment is based on the environment introduced by Schulman, Moritz, Levine, Jordan and Abbeel in \cite{schulman2015high}. The ant is a 3D robot consisting of one torso (free rotational body) with four legs attached to it with each leg having two body parts. The goal is to coordinate the four legs to move in the forward (right) direction by applying torques on the eight hinges connecting the two body parts of each leg and the torso (nine body parts and eight hinges). 

\textbf{Swimmer} \quad This environment corresponds to the Swimmer environment described in Rémi Coulom’s PhD thesis \cite{coulom2002reinforcement}. The environment aims to increase the number of independent state and control variables as compared to the classic control environments. The swimmers consist of three or more segments (’links’) and one less articulation joints (’rotors’) - one rotor joint connecting exactly two links to form a linear chain. The swimmer is suspended in a two dimensional pool and always starts in the same position (subject to some deviation drawn from an uniform distribution), and the goal is to move as fast as possible towards the right by applying torque on the rotors and using the fluids friction.

\subsection{Gym-Robotics Environment Tasks}

\textbf{AntMaze} \quad This environment was refactored from the D4RL repository, introduced by Justin Fu, Aviral Kumar, Ofir Nachum, George Tucker, and Sergey Levine in \cite{fu2020d4rl}. The task in the environment is for an ant-agent to reach a target goal in a closed maze. The ant is a 3D robot consisting of one torso (free rotational body) with four legs attached to it with each leg having two body parts. The goal is to reach a target goal in a closed maze by applying torques on the eight hinges connecting the two body parts of each leg and the torso (nine body parts and eight hinges).

\textbf{Fetch} \quad This environment was introduced in \cite{plappert2018multi}. The robot is a 7-DoF Fetch Mobile Manipulator with a two-fingered parallel gripper. The robot is controlled by small displacements of the gripper in Cartesian coordinates and the inverse kinematics are computed internally by the MuJoCo framework. The gripper is locked in a closed configuration in order to perform the push task. The task is also continuing which means that the robot has to maintain the block in the target position for an indefinite period of time. Notably, in ``FetchPush" or ``FetchSlide" when the absolute value of the reward is smaller than 0.05, namely $|r| < 0.05$, we consider the task to be terminated. This is because the returned dense reward is the negative Euclidean distance between the achieved goal position and the desired goal.

\textbf{Adroit Hand} \quad This environment was introduced in \cite{rajeswaran2017learning}. The environment is based on the Adroit manipulation platform, a 28 degree of freedom system which consists of a 24 degrees of freedom ShadowHand and a 4 degree of freedom arm. Notably, in ``AdroitDoor" when the absolute value of the reward is greater than 20, namely $|r| > 20$, we consider the task to be terminated. This is because a total positive reward of 20 is added if the door hinge is opened more than 1.35 radians. In ``AdroitHammer" when the absolute value of the reward is greater than 25, namely $|r| > 25$, we consider the task to be terminated. This is because a total positive reward of 25 is added if the euclidean distance between both body frames is less than 0.02 meters.

For more information about the tasks, please refer to the official document and the links are presented in Appendix~\ref{appendix:prompts}.

\subsection{Toy Example Settings}\label{appendix:toy_example}

This subsection introduces the settings of the toy example in Figure~\ref{fig:demo}.

\textbf{Task Description} \quad We employ the tailored interface of the \texttt{PointMaze} environment within Gym-Robotics\cite{de2023gymnasium}. As illustrated in Figure~\ref{fig:appendix_demo}, the maze has dimensions of $10\times10$, and the agent commences its navigation from the bottom-left corner, aiming to reach the top-right corner.

\begin{figure}[h]
	\centering
	\includegraphics[width=.2\textwidth]{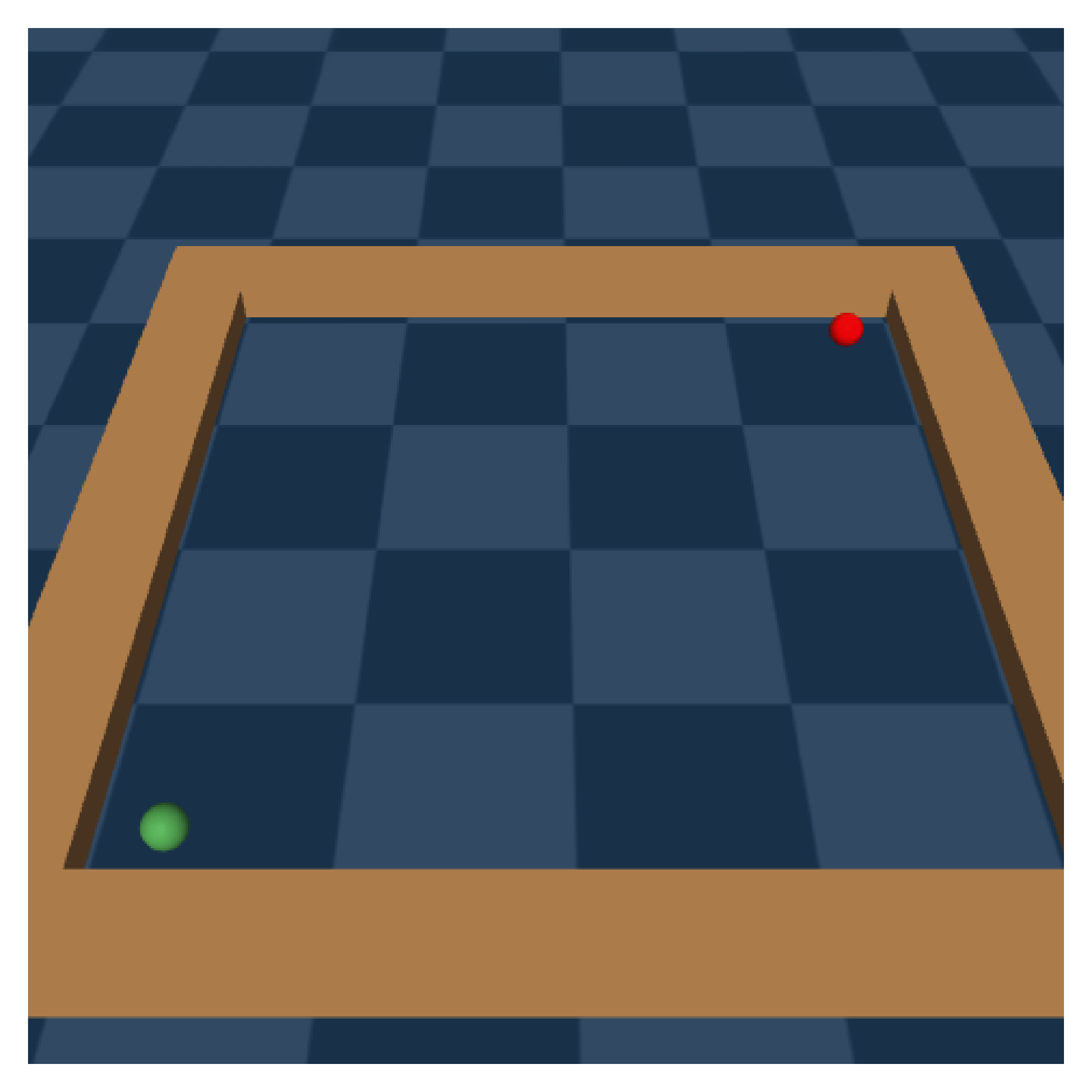}
	\caption{Demonstration of the customized maze in `Human' mode render.}
	\label{fig:appendix_demo}
\end{figure}

\textbf{Observation Space} \quad The source observation only contains an array of shape (4, ). The entire observation space is continuous, where obs[0] and obs[1] denotes the x, y coordinates of the agent's current position, while obs[2] and obs[3] represent the x, y coordinates of the target location.

\textbf{Action Space} \quad The source action only contains an array of shape (2, ). The entire action space is continuous, where action[0] and action[1] denote the linear force in the x, y direction to the agent.

\textbf{Rewards} \quad We use dense rewards. The returned reward is the negative Euclidean distance between the current agent's location and the target location. 

\textbf{Training Parameters} We use TD3~\cite{fujimoto2018addressing} as the RL algorithm. The values of hyperparameters for TD3\cite{fujimoto2018addressing} are derived from their original implementation\footnote{https://github.com/sfujim/TD3}. In Figure~\ref{fig:demo-a} `\textbf{1} Iteration', the `Episode Return' is normalized to 0-100.

\section{Additional Experiments} \label{appen:futher_ablations}

\subsection{Could LESR work if it uses the LLM-coded reward function as the only reward signal?}

We have conducted several experiments in Mujoco tasks. \textbf{Directly Intrinsic} means directly using the best state representation and intrinsic reward codes $\mathcal{F}_{best}, \mathcal{G}_{best}$ to train without extrinsic reward. \textbf{LESR w/o ER} means throughout the entire iterative process of LESR, extrinsic rewards are entirely removed. \textbf{LESR w/o LC} means throughout the entire iterative process of LESR, Lipschitz constant is entirely removed. Results under five random seeds are demonstrated in the following Table~\ref{table:only_intrinsic}. 

\begin{table}[h]
	\centering
	\caption{Experiments of using the LLM-coded reward function as the only reward signal, while dropping the external reward information.}
	\label{table:only_intrinsic}
	\resizebox{.8\columnwidth}{!}{
		\begin{tabular}{c|c|c|c|c|c}
			\toprule
			\diagbox{\textbf{Environments}}{\textbf{Algorithm}} & \textbf{TD3} & \textbf{Directly Intrinsic} & \textbf{LESR w/o ER} & \textbf{LESR w/o LC} & \textbf{LESR(Ours)} \\
			\midrule
			HalfCheetah & 9680.2±1555.8 & 8919.9±1761.9 & 10252.8±277.1 & 10442.9±304.4 & \textbf{10614.2±510.8} \\
            Hopper & 3193.9±507.8 & 3358.4±76.6 &	\textbf{3408.4±144.0} & 3362.3±109.4 & \textbf{3424.8±143.7}\\
            Walker2d & 3952.1±445.7 & 1924.3±969.3 & 3865.8±142.5 & \textbf{4356.2±335.3} & \textbf{4433.0±435.3} \\
            Ant & 3532.6±1265.3 & 3518.0±147.6 & \textbf{4779.1±30.4} & 3242.3±459.9 & 4343.4±1171.4 \\
            Swimmer-v3 & 84.9±34.0 & 25.2±4.1 & 51.9±0.1 & 116.9±6.3 & \textbf{164.2±7.6} \\
			\bottomrule
	\end{tabular} }
\end{table}

\textbf{Directly Intrinsic}: When we directly use $\mathcal{F_{best}}, \mathcal{G_{best}}$ to train without extrinsic reward. It's observed that the performance of our method is subject to perturbations but remains relatively stable, achieving comparable results to TD3 in environments such as HalfCheetah, Hopper, and Ant.

\textbf{LESR w/o ER}: Particularly, when we iterate without extrinsic rewards throughout the process, LESR demonstrates superior final performance over TD3 in most tasks.

\textbf{LESR w/o LC}: In experiments where the Lipschitz constant is removed, there is a performance decrease compared to LESR, yet still outperforming TD3, further indicating the effectiveness of utilizing the Lipschitz constant for feedback in our approach.

\subsection{Two New Task Decriptions}

To further validate the generalization capability of LESR, we conduct experiments on two entirely new task descriptions: (1) requiring the walker agent to learn to jump in place, and (2) requiring the walker agent to learn to stand with its legs apart. We abstain from any iterations or feedback, training solely on the $\mathcal{F}, \mathcal{G}$ produced in the first round of LESR, and only for 300k time steps. We generate the final GIFs in the github repository of LESR\footnote{https://github.com/thu-rllab/LESR}, and from the GIFs on the webpage, it is observed that the walker can perform the jumping and standing with legs apart actions as per the task descriptions, which further highlights the generalization significance of LESR. 

\subsection{Different Choices of Lipschitz constant}

As mentioned in Section~\ref{sec:lip_feedback}, apart from estimating the Lipschitz constant computed independently for each state dimension concerning the extrinsic reward, we can also employ the discouted return~(\textbf{LESR(DR)}) or spectral norm~(\textbf{LESR(SN)}) instead of the extrinsic reward. We have conducted experimental evaluations on these various feedback signals in Mujoco tasks in Table~\ref{table:different_choices}. Note: In this setting the sample count $K$ is set to 3. 

\begin{table}[h]
	\centering
	\caption{Various choices of Lipschitz constant estimation as feedback signals to LLM. \textbf{LESR(DR)}: LESR with Discounted Return. \textbf{LESR(SN)}: LESR with Spectral Norm.}
	\label{table:different_choices}
	\resizebox{.7\columnwidth}{!}{
		\begin{tabular}{c|c|c|c|c}
			\toprule
			\diagbox{\textbf{Environments}}{\textbf{Algorithm}} & \textbf{TD3} & \textbf{LESR(DR)} & \textbf{LESR(SN)} & \textbf{LESR(Ours)} \\
			\midrule
			Ant & 3532.6±1265.3 & 4451.0±1532.7 & 3649.2±1709.3 & \textbf{5156.4±267.3} \\
            HalfCheetah & 9680.2±1555.8 & 8926.1±1487.2 & \textbf{10830.8±340.6} & 9740.1±761.0 \\
            Hopper & 3193.9±507.8 & \textbf{3405.6±482.3} & 3179.4±592.5 & 3370.9±135.8 \\
            Swimmer & 84.9±34.0 & 60.1±10.5 & 49.7±1.9 & \textbf{142.0±7.6} \\
			\bottomrule
	\end{tabular} }
\end{table}

It's illustrated that both \textbf{LESR(DR)} and \textbf{LESR(SN)} are feasible for estimating the Lipschitz constant. However, LESR demonstrates comparatively more stable performance with lower variance in the context of dense reward environments. Nonetheless, this does not discount the utility of employing discounted return and spectral norm estimation methods. Specifically, in scenarios involving discontinuous rewards, the necessity of employing LESR(DR) and LESR(SN) becomes apparent.

\subsection{More Training Details of LESR}\label{appendix:training_curve}
We have presented the performance data for a limited scope of 300k environment interaction training steps in Table~\ref{tab:300k_run}. The results demonstrate that LESR exhibits superior performance within significantly fewer training steps, excelling comprehensively across all tasks in comparison to the RL baseline algorithm TD3, thereby manifesting heightened sample efficiency.

We have also presented the final evaluating curves in Figure~\ref{fig:evaluate_all} using all tasks' best state representation and intrinsic reward functions $\mathcal{F}_{best}, \mathcal{G}_{best}$. This further substantiates that our approach outperforms the baseline in terms of both sample efficiency and ultimate performances.
 
In Figure~\ref{fig:iteration} we have presented performance improments compared with the baseline using every iteration's best state representation and intrinsic reward functions $\mathcal{F}_{best}, \mathcal{G}_{best}$. It is demonstrated that the feedback part plays a critical role in our method. The final performance demonstrated a gradual amelioration across successive iterations.

\subsection{Semantic Analysis}\label{appendix:semantic}
Here are the state representation functions for the \textit{Swimmer} task across four random seeds.
The state representations in this task can be categorized into five groups, whose names are derived from the comments provided by the LLM: \textit{cosine or sine of the angles~(c1)}, \textit{relative angles between adjacent links~(c2)}, \textit{kinetic energy~(c3)}, \textit{distance moved~(c4)}, and \textit{sum of torques applied~(c5)}.

\begin{mybox}[frametitle=State representation functions for the \textit{Swimmer} task]
\textbf{$seed_1$}:
\begin{lstlisting}[language={Python}]
def revise_state(s):
    x_velocity_squared = s[3]**2
    total_angular_velocity = np.sum(s[5:]**2)
    cos_angles = np.cos(s[:3])
    sin_angles = np.sin(s[:3])
    relative_angles = np.diff(s[1:3])
    updated_s = np.concatenate((s, [x_velocity_squared, total_angular_velocity], cos_angles, sin_angles, relative_angles))
    return updated_s
\end{lstlisting}

\textbf{$seed_2$}:
\begin{lstlisting}[language={Python}]
def revise_state(s):
    if len(s) > 3:
        relative_angles = [s[i+1] - s[i] for i in range(1, len(s)-3)]
    else:
        relative_angles = []
    kinetic_energy = 0.5 * (s[3]**2 + s[4]**2) + 0.5 * sum(s[5:]**2)
    distance_moved = s[3]
    sum_torques = sum(abs(s[5:]))
    updated_s = np.concatenate((s, relative_angles, [kinetic_energy, distance_moved, sum_torques]))
    return updated_s
\end{lstlisting}

\textbf{$seed_3$}:
\begin{lstlisting}[language={Python}]
def revise_state(s):
    relative_angle_1_2 = s[1] - s[2] 
    total_kinetic_energy = 0.5 * (s[5]**2 + s[6]**2 + s[7]**2) 
    mass = 1.0
    length = 1.0
    gravity = 9.81
    potential_energy = mass * gravity * length * np.sin(s[0])
    dt = 1.0 
    distance_x = s[3] * dt
    updated_s = np.concatenate((s, [relative_angle_1_2, total_kinetic_energy, potential_energy, distance_x]))
    return updated_s
\end{lstlisting}

\textbf{$seed_4$}:
\begin{lstlisting}[language={Python}]
def revise_state(s):
    cos_angles = np.cos(s[:3])
    sin_angles = np.sin(s[:3])
    relative_angles = np.diff(s[:3])
    kinetic_energy = 0.5 * (s[3]**2 + s[4]**2) + 0.5 * np.sum(s[5:]**2)
    distance_moved = s[3]
    updated_s = np.concatenate((s, cos_angles, sin_angles, relative_angles, [kinetic_energy, distance_moved]))
    return updated_s
\end{lstlisting}

\end{mybox}

\begin{figure}[htbp]
	\centering
	\includegraphics[width=\textwidth]{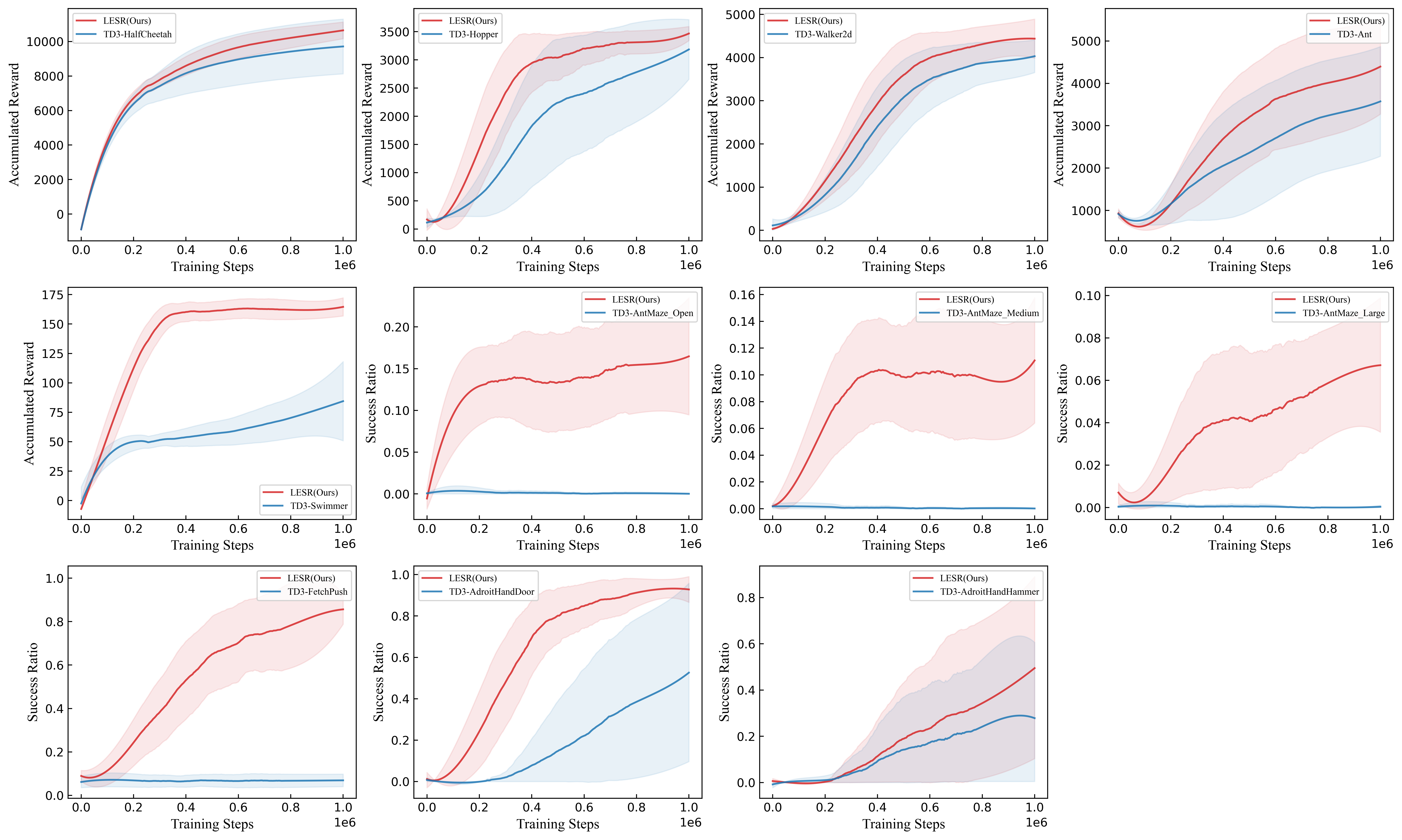}
	\caption{All tasks' final evaluating curves using the best state representation and intrinsic reward functions $\mathcal{F}_{best}, \mathcal{G}_{best}$ over 5 random seeds. We have smoothed the curves using avgol\_filter in `scipy'. }
	\label{fig:evaluate_all}
\end{figure} 

\begin{figure}[htbp]
	\centering
	\begin{subfigure}[t]{0.408\linewidth}
		\centering
		\includegraphics[width=\linewidth]{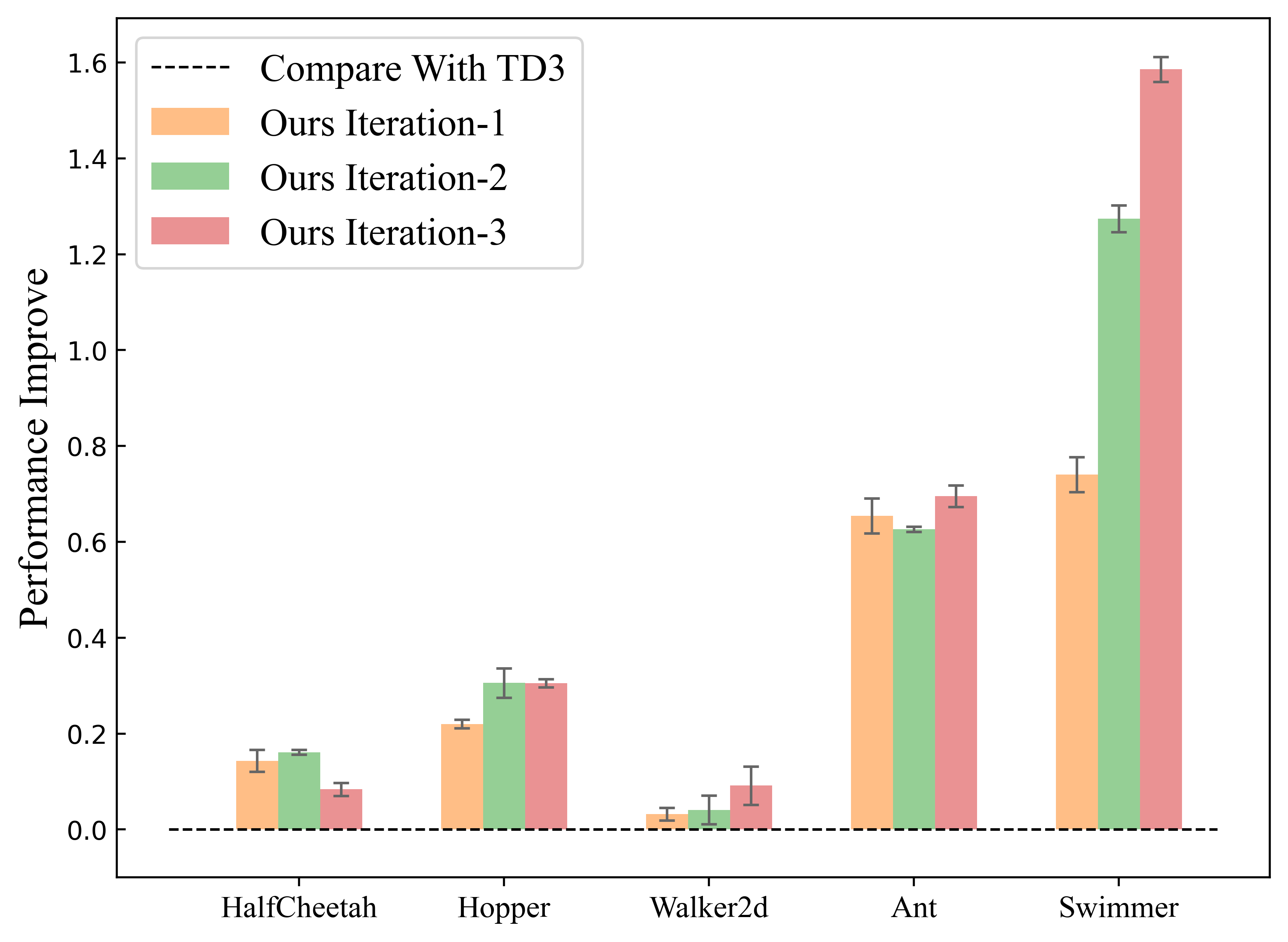}
		\caption{Mujoco Tasks}
	\end{subfigure}%
	\begin{subfigure}[t]{0.592\linewidth}
		\centering
		\includegraphics[width=\linewidth]{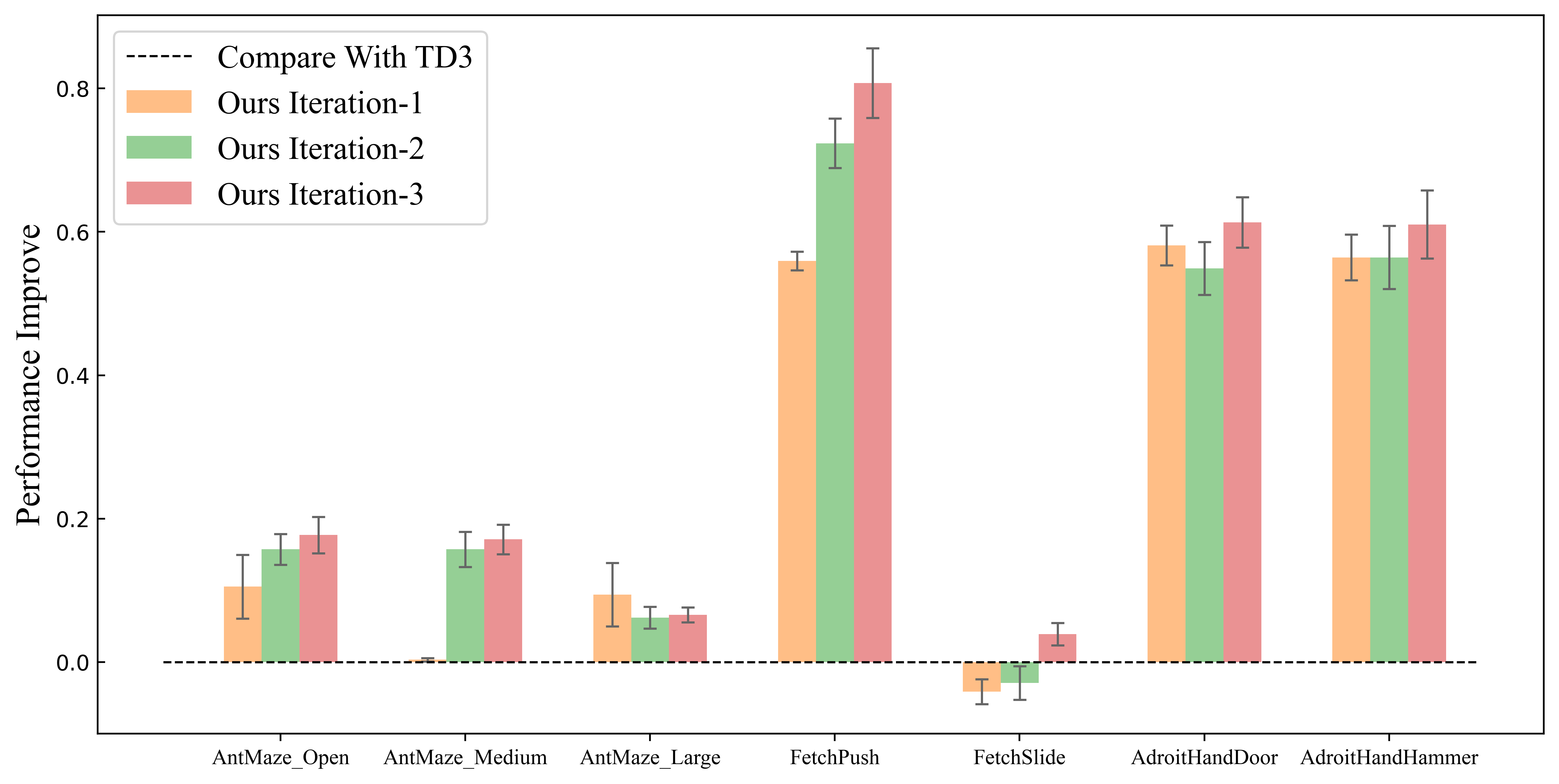}
		\caption{Gym-Robotics Tasks}
	\end{subfigure}%
	\caption{Performance improments compared with the baseline TD3 during every iteration. We use every iteration's best state representation and intrinsic reward functions $\mathcal{F}_{best}, \mathcal{G}_{best}$. Performance improments are calculated by $\frac{accumulated\_rewards - baseline}{baseline}$, while in Gym-Robotics, it is determined as $success\_rate - baseline$.}
	\label{fig:iteration}
\end{figure}

\begin{table*}[t]
	\caption{Final Performance on Mujoco and Gym-robotics environments over only \textbf{300k} environment interaction training steps. \textbf{Ours w/o IR:} without intrinsic reward. \textbf{Ours w/o SR:} without state representation. \textbf{Ours w/o FB:} without feedback. More details about the ablation study are elucidated in Section~\ref{sec:ablation}. \textbf{RPI} = \texttt{(accumulated\_rewards - baseline)/baseline}: relative performance improvement. \textbf{PI} = \texttt{success\_rate - baseline}: success rate improvement. This distinction arises because the performance is represented by the accumulative rewards in Mujoco, whereas in Gym-Robotics, it is measured by the success rate. The "mean$\pm$std" denotes values computed across five random seeds.}
	\label{tab:300k_run}
	\resizebox{\textwidth}{!}{
		\begin{tabular}{c|c|cc|cc|cc|cc|cc}
			\toprule
			\diagbox{\textbf{Environments}}{\textbf{Algorithm}} & \textbf{TD3} & \textbf{EUREKA} & \textbf{RPI} & \textbf{Ours w/o IR} & \textbf{RPI} & \textbf{Ours w/o SR} & \textbf{RPI} & \textbf{Ours w/o FB} & \textbf{RPI} &\textbf{LESR(Ours)} & \textbf{RPI} \\
			\midrule
			HalfCheetah & 7288.3$\pm$842.4 & 7299.1$\pm$344.2 & 0\% & 6985.7$\pm$921.8 & -4\% & 6924.5$\pm$483.9 & -5\% & 5391.6$\pm$3221.2 & -26\% & \textbf{7639.4$\pm$464.1} & \textbf{5\%} \\
			Hopper & 921.9$\pm$646.8 & 2132.1$\pm$0.0 & 131\% & 2509.0$\pm$619.8 & 172\% & 1638.0$\pm$1031.3 & 78\% & 1993.4$\pm$956.0 & 116\% & \textbf{2705.0$\pm$623.2} & \textbf{193\%} \\
			Walker2d & 1354.5$\pm$734.1 & 512.9$\pm$132.3 & -62\% & 1805.5$\pm$948.1 & 33\% & 962.1$\pm$388.4 & -29\% & 1270.2$\pm$659.5 & -6\% & \textbf{1874.8$\pm$718.2} & \textbf{38\%} \\
			Ant & 1665.2$\pm$895.5 & 1138.2$\pm$0.0 & -32\% & \textbf{2409.9$\pm$575.2} & \textbf{45\%} & 1967.9$\pm$1090.3 & 18\% & 2025.0$\pm$910.5 & 22\% & 1915.3$\pm$885.5 & 15\% \\
			Swimmer & 49.9$\pm$5.2 & 32.3$\pm$0.0 & -35\% & 57.5$\pm$4.2 & 15\% & 65.8$\pm$24.0 & 32\% & 51.3$\pm$5.0 & 3\% & \textbf{150.9$\pm$9.5} & \textbf{203\%} \\
			\textbf{Mujoco Relative Improve Mean} & - & - & 1\% &  - & 52\% &  - & 19\% &  - & 22\% &  - & \textbf{91\%} \\
			\midrule
			- & - & - & \textbf{PI} & - & \textbf{PI} &  - & \textbf{PI} &  - & \textbf{PI} &  - & \textbf{PI} \\
			\midrule
			AntMaze\_Open & 0.0$\pm$0.0 & 0.12$\pm$0.03 & 12\% & 0.0$\pm$0.0 & 0\% & \textbf{0.19$\pm$0.05} & \textbf{19\%} & 0.15$\pm$0.06 & 15\% & 0.15$\pm$0.06 & 15\% \\
			AntMaze\_Medium & 0.0$\pm$0.0 & 0.01$\pm$0.0 & 1\% & 0.0$\pm$0.0 & 0\% & 0.06$\pm$0.03 & 6\% & 0.01$\pm$0.01 & 1\% & \textbf{0.11$\pm$0.05} & \textbf{11\%} \\
			AntMaze\_Large & 0.0$\pm$0.0 & 0.0$\pm$0.0 & 0\% & 0.0$\pm$0.0 & 0\% & 0.03$\pm$0.02 & 3\% & 0.01$\pm$0.02 & 1\% & 0.04$\pm$0.03 & 4\% \\
			FetchPush & 0.06$\pm$0.03 & 0.08$\pm$0.03 & 2\% & 0.29$\pm$0.17 & 23\% & 0.06$\pm$0.03 & 0\% & 0.26$\pm$0.14 & 19\% & \textbf{0.41$\pm$0.15} & \textbf{35\%} \\
			AdroitHandDoor & 0.0$\pm$0.01 & 0.10$\pm$0.07 & 10\% & 0.12$\pm$0.2 & 12\% & 0.02$\pm$0.03 & 2\% & 0.04$\pm$0.07 & 4\% & \textbf{0.38$\pm$0.26} & \textbf{38\%} \\
			AdroitHandHammer & 0.02$\pm$0.02 & 0.04$\pm$0.05 & 2\% & 0.0$\pm$0.01 & -1\% & 0.0$\pm$0.01 & -1\% & \textbf{0.07$\pm$0.11} & \textbf{5\%} & 0.01$\pm$0.01 & -1\% \\
			\textbf{Gym-Robotics Improve Mean} & - & - & 5\% &  - & 6\% &  - & 5\% &  - & 8\% &  - & \textbf{17\%} \\
			\bottomrule
		\end{tabular} 
	}
\end{table*}

\section{Future Work}

\textbf{Methodology framework viability} Our methodology framework in LESR remains viable for image tasks, where Vision-Language Models (VLMs)\cite{radford2021learning, kim2021vilt} can be employed to extract semantic features from images, followed by further processing under the LESR framework. We anticipate utilizing VLMs for future research.

\textbf{General applicability beyond symbolic environments} While our primary focus lies on symbolic environments, our method extends beyond them. LESR serves as a general approach for leveraging large models to generate Empowered State Representations, offering potential applicability across various environments.

\textbf{Offline reinforcement learning} LESR is also feasible for offline reinforcement learning scenarios\cite{qu2024hokoff, shao2024counterfactual, mao2023supported, mao2024supported}. The LESR framework is versatile and not limited to online RL. In the future, we aim to explore various applications and possibilities.

Within LESR, we utilize LLMs to generate Empowered State Representations, showcasing their effectiveness in enhancing the Lipschitz continuity of value networks in reinforcement learning. Our experimental results, along with supplementary theorems, validate these advantages. We believe that LESR holds promise in inspiring future research endeavors.

\section{Hyperparameters}\label{appendix:hypers}

We have listed the hyperparameters of our algorithm in Table~\ref{table:hyper}, which encompasses the parameters of the training process, algorithm and optimizer settings. Notably, for PPO and SAC, we adopt the code in \url{https://github.com/Lizhi-sjtu/DRL-code-pytorch}. For EUREKA, to ensure comparative fairness, we utilized the same hyperparameters as LESR, namely the sample count $K$ and the iteration count $I$.

\begin{table}[htbp]
	\centering
	\caption{Hyperparameters for our algorithm. The values of hyperparameters for TD3\cite{fujimoto2018addressing} are derived from their original implementation.}
	\resizebox{\textwidth}{!}{
	\begin{tabular}{lcc}
		\toprule
		\textbf{Hyperparameters Name} & \textbf{Explaination} & \textbf{Value} \\
		\midrule
		$K$ & How many sample to generate during one iteration & 6\\
		$I$ & How many iterations & 3\\
		$N_{small}$ & Total training timesteps for iteration & 8e5\\
		$N$ & Final total training timesteps for $\mathcal{F}_{best}, \mathcal{G}_{best}$ & 1e6\\
		$w$ & weight of intrinsic reward & 0.02 for Mujoco tasks, 0.2 for Gym-Robotics tasks\\
		\bottomrule
	\end{tabular}%
	}
	\label{table:hyper}%
\end{table}%


\end{document}